\documentclass[english,letterpaper]{article}
\usepackage{mathptmx}
\usepackage{helvet}
\usepackage{courier}
\usepackage[utf8]{inputenc}
\usepackage{array}
\usepackage{refstyle}
\usepackage{booktabs}
\usepackage{multirow}
\usepackage{amsmath}
\usepackage{amsthm}
\usepackage{amssymb}
\usepackage{graphicx}
\usepackage[authoryear]{natbib}
\usepackage{microtype}

\makeatletter

\AtBeginDocument{\providecommand\Figref[1]{\ref{Fig:#1}}}
\AtBeginDocument{\providecommand\Secref[1]{\ref{Sec:#1}}}
\AtBeginDocument{\providecommand\Tabref[1]{\ref{Tab:#1}}}
\AtBeginDocument{\providecommand\secref[1]{\ref{sec:#1}}}
\AtBeginDocument{\providecommand\tabref[1]{\ref{tab:#1}}}
\AtBeginDocument{\providecommand\Eqref[1]{\ref{Eq:#1}}}
\AtBeginDocument{\providecommand\thmref[1]{\ref{thm:#1}}}
\AtBeginDocument{\providecommand\Appendixref[1]{\ref{Appendix:#1}}}
\AtBeginDocument{\providecommand\appendixref[1]{\ref{appendix:#1}}}
\AtBeginDocument{\providecommand\eqref[1]{\ref{eq:#1}}}
\AtBeginDocument{\providecommand\figref[1]{\ref{fig:#1}}}
\AtBeginDocument{\providecommand\lemref[1]{\ref{lem:#1}}}
\AtBeginDocument{\providecommand\Lemref[1]{\ref{Lem:#1}}}
\AtBeginDocument{\providecommand\Thmref[1]{\ref{Thm:#1}}}
\providecommand{\tabularnewline}{\\}
\RS@ifundefined{subsecref}
  {\newref{subsec}{name = \RSsectxt}}
  {}
\RS@ifundefined{thmref}
  {\def\RSthmtxt{theorem~}\newref{thm}{name = \RSthmtxt}}
  {}
\RS@ifundefined{lemref}
  {\def\RSlemtxt{lemma~}\newref{lem}{name = \RSlemtxt}}
  {}

\theoremstyle{plain}
\newtheorem{thm}{\protect\theoremname}
\theoremstyle{plain}
\newtheorem{lem}[thm]{\protect\lemmaname}
\theoremstyle{plain}
\newtheorem*{thm*}{\protect\theoremname}

\usepackage{style/AnonymousSubmission/LaTeX/aaai25}

\usepackage[hyphens]{url}  
\usepackage{caption} 
\usepackage{lipsum}

\newref{sec}{%
name = \RSsectxt,
names = \RSsecstxt,
Name = \RSSectxt,
Names = \RSSecstxt,
refcmd = {\ref{#1}},
rngtxt = \RSrngtxt,
lsttwotxt = \RSlsttwotxt,
lsttxt = \RSlsttxt}

\newref{subsec}{%
name = \RSsectxt,
names = \RSsecstxt,
Name = \RSSectxt,
Names = \RSSecstxt,
refcmd = {\ref{#1}},
rngtxt = \RSrngtxt,
lsttwotxt = \RSlsttwotxt,
lsttxt = \RSlsttxt}

\newref{appendix}{%
name = \RSappendixname,
names = \RSappendicesname,
Name = \RSAppendixname,
Names = \RSAppendicesname,
refcmd = {\ref{#1}},
rngtxt = \RSrngtxt,
lsttwotxt = \RSlsttwotxt,
lsttxt = \RSlsttxt}

\newref{thm}{%
name = \RSthmtxt,
names = \RSthmstxt,
Name = \RSThmtxt,
Names = \RSThmsstxt,
refcmd = {\ref{#1}},
rngtxt = \RSrngtxt,
lsttwotxt = \RSlsttwotxt,
lsttxt = \RSlsttxt}

\newref{lem}{%
name = \RSlemtxt,
names = \RSlemstxt,
Name = \RSLemtxt,
Names = \RSLemsstxt,
refcmd = {\ref{#1}},
rngtxt = \RSrngtxt,
lsttwotxt = \RSlsttwotxt,
lsttxt = \RSlsttxt}

\newref{eq}{%
name = \RSeqtxt,
names = \RSeqstxt,
Name = \RSEqtxt,
Names = \RSEqsstxt,
refcmd = {(\ref{#1})},
rngtxt = \RSrngtxt,
lsttwotxt = \RSlsttwotxt,
lsttxt = \RSlsttxt}

\nocopyright

\makeatother

\usepackage{babel}
\providecommand{\lemmaname}{Lemma}
\providecommand{\theoremname}{Theorem}

\begin{document}
\title{A Robust Prototype-Based Network with Interpretable RBF Classifier
Foundations}
\author{Sascha Saralajew,$^{1}$ Ashish Rana,$^{1}$ Thomas Villmann,$^{2}$
and Ammar Shaker$^{1}$}
\affiliations{$^{1}$NEC Laboratories Europe, Germany\\
$^{2}$University of Applied Sciences Mittweida, Germany\\
\{sascha.saralajew, ashish.rana, ammar.shaker\}@neclab.eu, villmann@hs-mittweida.de}

\maketitle
\begin{abstract}
Prototype-based classification learning methods are known to be inherently
interpretable. However, this paradigm suffers from major limitations
compared to deep models, such as lower performance. This led to the
development of the so-called deep Prototype-Based Networks (PBNs),
also known as prototypical parts models. In this work, we analyze
these models with respect to different properties, including interpretability.
In particular, we focus on the Classification-by-Components (CBC)
approach, which uses a probabilistic model to ensure interpretability
and can be used as a shallow or deep architecture. We show that this
model has several shortcomings, like creating contradicting explanations.
Based on these findings, we propose an extension of CBC that solves
these issues. Moreover, we prove that this extension has robustness
guarantees and derive a loss that optimizes robustness. Additionally,
our analysis shows that most (deep) PBNs are related to (deep) RBF
classifiers, which implies that our robustness guarantees generalize
to shallow RBF classifiers. The empirical evaluation demonstrates
that our deep PBN yields state-of-the-art classification accuracy
on different benchmarks while resolving the interpretability shortcomings
of other approaches. Further, our shallow PBN variant outperforms
other shallow PBNs while being inherently interpretable and exhibiting
provable robustness guarantees. 
\end{abstract}

\section{Motivation and Context\label{sec:Motivation-and-Context}}

Two principal streams exist in the field of explainable machine learning:
(1) post-processing methods (post-hoc approaches) that try to explain
the prediction process of an existing model, such as LIME and SHAP
(see \citealp{Marcinkevics2023}, for an overview), and (2) the design
of machine learning methods with inherently interpretable prediction
processes \citep{Rudin2019}. While the former could create non-faithful
explanations due to only approximating the output distribution of
a black box model without explaining its internal logic, it is claimed
that inherently interpretable methods always generate faithful explanations
\citep{Rudin2019}. According to \citet{Molnar2022}, a model is called
\textit{interpretable} if its behavior and predictions are understandable
to humans. Moreover, when the provided explanations lead to a correct
interpretation of the model, this interpretation enriches the user
(or developer) with an understanding of how the model works, how it
can be fixed or improved, and whether it can be trusted \citep{Ribeiro2016}.

A well-known category of interpretable models for classification tasks
is (shallow) Prototype-Based Networks (PBN) such as LVQ \citep[e.\,g.,][]{Biehl2016}.
These models are interpretable because (1) the learned class-specific
prototypes\footnote{Usually, prototypes are class-specific, and components, centers, or
centroids are class-unspecific.} are either from the input space or can be easily mapped to it; belonging
to the input space helps summarize the differentiating factors of
the input data and provides trusted exemplars for each class, (2)
the dissimilarity computations are given by human comprehensible equations
such that differences between inputs and learned prototypes can be
understood, (3) the classification rule based on the dissimilarities
is intelligible (e.\,g., winner-takes-all principle); see \citet{Bancos2020}
for an interpretability application. Despite being interpretable,
these models also face limitations: (1) The number of parameters becomes
large on complex data since the prototypes are class-specific and
are defined in the input space.\footnote{A ResNet50 on ImageNet has 26\,M parameters, whereas an LVQ model
with one prototype per class has 150\,M.} (2) The classification performance is behind that of deep neural
architectures as the dissimilarity functions and the classification
rules are straightforward to ensure interpretability \citep{Villmann2017}. 

To fix these limitations, researchers investigated the integration
of prototype-based classification heads with deep neural feature extractors
to build deep interpretable PBNs and designed numerous architectures
such as ProtoPNet \citep{Chen2019}, ProtoPool \citep{Rymarczyk2022},
CBC (Classification-By-Components; \citealp{Saralajew2019}), and
PIPNet \citep{Nauta2023}. The generated results of these models are
impressive as they achieve state-of-the-art classification accuracy
on fine-grained image classification, and some show a good performance
in rejecting Out-Of-Distribution (OOD) examples (e.\,g., PIPNet).
The high-level structure of these models follows the same principles
(see \Figref{General-architecture}): (1) embedding of the input data
in a latent space by a Neural Network (NN), denoted as feature extractor
backbone; (2) measuring the dissimilarity (or similarity) between
the embedding and the latent prototypes; (3) prediction computation
after aggregating the dissimilarities by a shallow model (realizes
the classification rule), denoted as classification head. In this
paradigm, the differences between the proposed architectures are often
subtle, such as imposing sparsity, the usage of negative reasoning,
and whether they can be used as a shallow model. Moreover, all architectures
are supposed to generate interpretable models. But is this genuinely
accurate?

In this paper, we investigate PBNs and make the following contributions:
\begin{enumerate}
\item We show that deep PBNs are related to deep RBF classifiers. Building
on this finding, we explain why these models are effective for OOD
detection.
\item We discuss why current \emph{deep} PBNs are not interpretable and
demonstrate how the interpretability level of the models varies between
the different architectures.
\item Building on CBCs and their relation to RBF networks, we design a prototype-based
classification head that can use negative reasoning in a sound probabilistic
way and fixes the interpretability issue of other heads.
\item We derive robustness bounds for our classification head (shallow PBN),
including a loss that provably optimizes robustness. Further, the
relation shown gives the first loss that optimizes the robustness
of RBF classifiers.
\end{enumerate}
The paper's outline is as follows: In \Secref{Review-of-Prototype-based},
we review deep PBNs and discuss their relation to RBF networks \citep{Broomhead1988}
and several properties. Based on the identified shortcomings, in \Secref{Classification-by-Components-Net},
we propose an extension of CBC so that the interpretability is sound
and negative reasoning is used. Additionally, we show that the shallow
version of this architecture has provable robustness guarantees. \Secref{Experiments}
presents the experimental evaluation of our claims. Finally, a discussion
and conclusion are presented.

\section{Review of Deep Prototype-based Networks\label{sec:Review-of-Prototype-based}}

\begin{figure}
\begin{centering}
\includegraphics[viewport=19bp 17bp 447bp 177bp,clip,width=1\columnwidth]{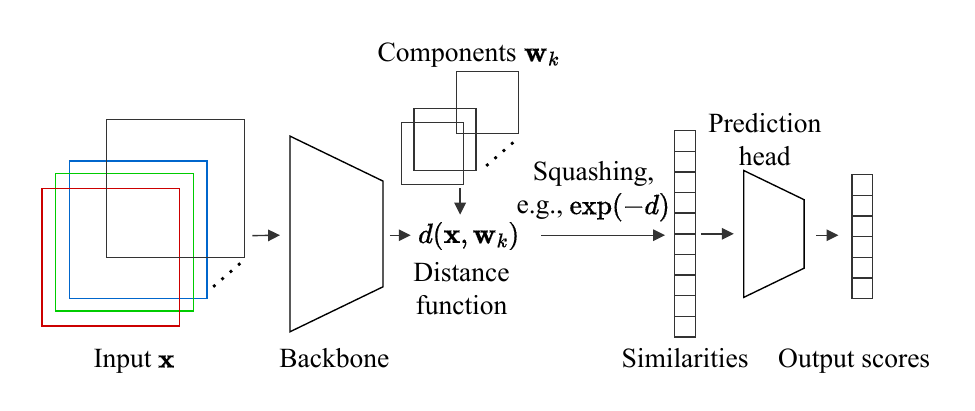}\caption{General architecture of deep PBNs.\label{fig:General-architecture}}
\par\end{centering}
\end{figure}
\begin{table*}[t]
\begin{centering}
\small{%
\begin{tabular*}{1\textwidth}{@{\extracolsep{\fill}}>{\centering}m{2.2cm}c>{\centering}m{1.1cm}c>{\centering}m{1.8cm}>{\centering}m{1cm}>{\centering}m{4cm}}
 & \textbf{Backbone} & \textbf{Latent Proto.} & \textbf{Similarity} & \textbf{Linear Layer Constraints} & \textbf{Single Loss} & \textbf{Main Contribution}\tabularnewline
\midrule
\midrule 
\textbf{LeNet5 } & single & yes & RBF & none & no & CNN with RBF head\tabularnewline
\addlinespace[3pt]
\textbf{ProtoPNet{*}} & single & \emph{yes{*}} & RBF (log) & $l_{1}$ reg. & no & (deep) NN with prototype classification head\tabularnewline
\addlinespace[3pt]
\textbf{CBC{*}} & Siamese{*} & no & RBF or ReLU-cosine & probabilistic & yes & negative/positive/indefinite reasoning\tabularnewline
\addlinespace[3pt]
\textbf{ Hier. ProtoPNet } & single & \emph{yes{*}} & RBF (log) & $l_{1}$ reg. & no & hierarchical classification\tabularnewline
\addlinespace[3pt]
\textbf{ProtoAttend{*}} & Siamese & no & relational attention & none & no & attention for prototype selection\tabularnewline
\addlinespace[3pt]
\textbf{ProtoTree} & single & yes{*} & RBF & (soft) tree{*} & yes & tree upon similarities\tabularnewline
\addlinespace[3pt]
\textbf{ProtoPShare} & single & yes{*} & RBF (log) & $l_{1}$ reg. & no & prototype sharing between classes\tabularnewline
\addlinespace[3pt]
\textbf{TesNet} & single & \emph{yes{*}} & dot-product & $l_{1}$ reg. & no & orthogonal prototypes\tabularnewline
\addlinespace[3pt]
\textbf{ Def. ProtoPNet} & single & \emph{yes{*}} & RBF (cosine) & $l_{1}$ reg. & no &  deformable prototypes (shift correction)\tabularnewline
\addlinespace[3pt]
\textbf{ProtoPool} & single & yes{*} & RBF (focal similarity) & $l_{1}$ reg. & no & differentiable prototype selection\tabularnewline
\addlinespace[3pt]
\textbf{PIPNet} & single & yes & softmax dot product & non-negative & no & self-supervised pre-training \tabularnewline
\addlinespace[3pt]
\textbf{LucidPPN} & multiple & yes & sigmoid dot product & average{*} & no & color and shape backbone\tabularnewline
\addlinespace[3pt]
\textbf{ProtoViT} & single & yes{*} & scaled sum of cosine & $l_{1}$ reg. & no & deformable prototypes through vision transformer\tabularnewline
\midrule
\textbf{Ours} & single{*} & yes & RBF or softmax dot product & probabilistic & yes & trainable priors and provable robustness\tabularnewline
\end{tabular*}}
\par\end{centering}
\caption{Characterization of existing architectures along the specified dimensions.
Note that the order is chronological. Methods that are not directly
based on the previously published methods are marked with an asterisk.
The asterisk in the remaining columns stands for the ability to omit
the feature extractor in Backbone, the usage of back-projection of
latent prototypes in Latent Prototype, and an alternative approach
for the output computation in Linear Layer Constraints (i.\,e., no
application of a linear layer with a regularization or constraint).
The italic typeface in Latent Prototype states that the prototypes
are class-specific.\label{tab:Overview-of-existing}}
\end{table*}
In the following section, we review the differences between deep PBNs
and show their relation to RBF networks. Thereafter, we discuss the
interpretability of these methods using the established relation.
Later, we explain why PBNs are suitable for OOD detection and analyze
the role of negative reasoning.

\paragraph{Differences between the architectures and their relation to RBF networks.}

\Figref{General-architecture} shows the general architecture of most
deep PBNs. We use the shown building blocks to characterize existing
approaches in \Tabref{Overview-of-existing} along the following dimensions:
\begin{itemize}
\item \textbf{Backbone:} Single, multiple, or Siamese feature extractor,
and whether the method has been tested without a feature extractor
(shallow model). 
\item \textbf{Latent prototypes:} Whether the prototypes are defined in
the input or the latent space and if they are back-projected to training
samples \citep{Chen2019}. This dimension also indicates if prototypes
are class-specific.
\item \textbf{Similarity:} The used similarity function. RBF refers to the
standard squared exponential kernel. If a different nonlinear function
is used to construct the RBF, it is specified in parenthesis. Note
that all RBFs use the Euclidean norm.
\item \textbf{Linear layer constraints:} The constraints on the final linear
prediction layer or the stated approach to compute the output if no
linear output layer is used. The $l_{1}$ regularization is only applied
to connections that connect similarity scores (slots, etc.) with incorrect
classes.
\item \textbf{Single loss term:} Whether multiple loss terms are used. 
\item \textbf{Main contribution:} The primary contribution of the proposed
architecture compared to previous work.
\end{itemize}
We identified the following architectures by reviewing top-tier venue
papers: LeNet5 \citep{LeCun1998}, ProtoPNet, CBC, Hierarchical ProtoPNet
\citep{Hase2019}, ProtoAttend \citep{Arik2020}, ProtoTree \citep{Nauta2021},
ProtoPShare \citep{Rymarczyk2022}, TesNet \citep{Wang2021}, Deformable
ProtoPNet \citep{Donnelly2022}, ProtoPool, and PIPNet. Moreover,
we added LucidPPN \citep{Pach2024} and ProtoViT \citep{Ma2024} as
it is the most recent publication in the field.

Considering \Figref{General-architecture}, we realize that the head
of a deep PBN is an RBF network if a linear layer is used for prediction.
Combined with a feature extractor, we obtain \emph{deep RBF networks}
\citep[e.\,g.,][]{Asadi2021}. Notably, the first deep PBN is LeNet5,
where RBF heads are used to measure the similarity between inputs
and the so-called ``model'' (prototype) of the class. Starting with
ProtoPNet, the existing architectures (see \Tabref{Overview-of-existing})
build on each other (except for CBC and ProtoAttend), and almost all
use an RBF network with some constraints or regularizers as classification
heads. Consequently, changes between the architectures are incremental,
and concepts persist for some time once introduced. Recently, researchers
abandoned the idea of back-projecting prototypes and started using
dot products instead of RBF functions, which implicitly defines prototypes
as convolutional filter kernels (PIPNet, LucidPPN). 

\paragraph{On the interpretability of deep PBNs.}

Using the definition of interpretability in \Secref{Motivation-and-Context}
and the relation to RBF networks, we discuss the interpretability
of deep PBNs. First, it should be noted that RBF networks and shallow
PBNs learn representations in the input space (centroids and prototypes,
respectively), and both use these representations to measure the (dis)-similarity
to given samples. At the same time, these two paradigms differ in
two aspects: (1) RBFs' usage of non-class specific centroids and (2)
PBNs' usage of the human-comprehensible winner-takes-all rule instead
of a linear predictor over the prototypes.

The first aspect overcomes the Limitation~(1) mentioned in \Secref{Motivation-and-Context}
without harming the interpretability. The second aspect poses a problem
for interpretation, which explains the lack of studies applying RBF
networks for interpretable machine learning. The problem starts with
the unconstrained weights in the linear layer, which lead to unbounded
and incomparable scores (e.\,g., it is unclear how to interpret a
high score or weight). Further, this could result in situations where
the closest (most similar) centroids do not contribute the most to
the classification score compared to less similar centroids that are
overemphasized by large weights. Hence, this breaks the paradigm that
the most similar centroids (or prototypes) define the class label.

What does this imply for deep PBNs? First, the interpretation of the
classification head suffers from the same difficulties as an RBF network
if no appropriate constraints are applied (e.\,g., the average computation
of LucidPPN). Therefore, the most similar prototypes for an input
do not necessarily define the class label. For example, PIPNet trained
on CUB \citep{Wah2011} uses average weights of 14.1 for class blue
jay and 8.7 for green jay. This indicates that PIPNet overemphasizes
small similarity values so that the interpretation of the influential
prototypes could be incorrect; further results in \secref{Experiments}.
Second, since the similarity is computed in a latent space defined
by a deep NN, it is unclear why two samples are close or distant due
to the black-box nature of deep NNs. Thus, it is \emph{misleading}
to denote a deep PBN as interpretable. In the best case, it can be
denoted as \emph{partially interpretable} as it gives insights into
the final classification step, assuming that the classification head
is well-designed. Note that the interpretability of these methods
is also questioned by others \citep[e.\,g.,][]{Hoffmann2021,Pazzani2022,Sacha2024,Wolf2024}.

\paragraph{On the OOD detection properties.}

In CBC (rejection of predictions), Hierarchical ProtoPNet (novel class
detection), ProtoAttend (OOD detection), and PIPNet (OOD detection),
it was shown that deep PBNs are suitable for identifying OOD samples.
This ability can be attributed to the RBF architecture if the model
is clearly related to RBF models. \citet{Hein2019} proved that RBF
networks produce low-confidence predictions when \emph{a given sample
is far away from all centroids (prototypes)} because the applied softmax
squashing enforces the predictions of all classes to be uniform. \Citet{Amersfoort2020}
built on this idea and empirically showed that deep feature extractors
with an RBF head and a winner-takes-all rule (so a deep PBN) can be
used for uncertainty estimation and, thus, OOD detection. The published
results for deep PBNs also confirm this property \citet{Amersfoort2020}
observed. Empirically, this property transfers beyond RBF-related
architectures, as architectures like PIPNet show a remarkable OOD
performance using a non-RBF similarity.

\paragraph{The role of negative reasoning.}

Positive reasoning is well-defined as retrieving evidence of a given
class from present features, but the literature does not reach a consensus
about negative reasoning. In CBC, it means the retrieval of evidence
from absent features. In contrast, in ProtoPNet, this refers to the
reduction of the final score due to a negative weight associated with
an active prototype. Other methods in the literature either penalize
negative reasoning (e.\,g., ProtoPNet) by a regularization term or
avoid it by a constraint (e.\,g., PIPNet); see the Constraints column
in \tabref{Overview-of-existing} . The challenge posed by negative
reasoning in these architectures is mainly about interpretation, as
it is not an intuitive reasoning principle of humans (according to
\citealp{Chen2019}) and complicates the explanation strategies. In
a notable contrast, in CBC, inspired by cognitive science results
\citep[e.\,g.,][]{Hsu2017}, the authors modeled negative reasoning
from a probabilistic perspective, making its interpretation mathematically
sound. For the remainder of the paper, we refer by negative reasoning
to the retrieval of evidence from features that have to be absent.

\section{Classification-by-Components Networks\label{sec:Classification-by-Components-Net} }

We now review the original CBC architecture and show its limitations.
Based on that, we propose our CBC---simply denoted as CBC and the
old version is denoted as \emph{original CBC}---that overcomes these
limitations and realizes a strong link to RBF networks. Then, we show
how a CBC can be learned efficiently and derive robustness lower bounds.

\paragraph{Review of the original CBC method.}

\begin{figure}
\begin{centering}
\includegraphics[viewport=0bp 2bp 235bp 88bp,clip,width=1\columnwidth]{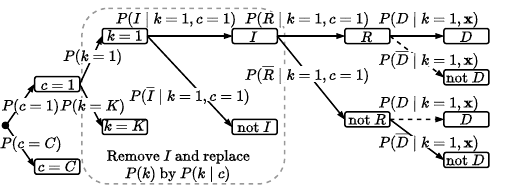}
\par\end{centering}
\caption{Probability tree diagram of the original CBC with the changes we propose
for our extension in the gray box.\label{fig:Probability-tree-diagram-CBC}}
\end{figure}
Components are the core concept of the original CBC, where a component
is a pattern that contributes to the classification process by its
presence (positive reasoning; the component must be close) or absence
(negative reasoning; the component must be far) without being tied
to a specific class label. A component can also abstain from the classification
process, which is called indefinite reasoning (modeled via importance).
The original CBC is based on a probability tree diagram to model the
interaction between the detection of components in input samples and
the usage of detection responses to model the output probability (called
reasoning). The probability tree, \Figref{Probability-tree-diagram-CBC},
employs five random variables: $c$, the class label; $k$, the component;
$I$, the importance of a component (binary); $R$, the requiredness
for reasoning (binary); $D$, the detection of a component (binary).
The probability tree constructs the following: $P\left(k\right)$,
the prior of the $k$-th component to appear; $P\left(I|k,c\right)$
and $P\left(R|k,c\right)$ are the importance and the requiredness
probabilities of the $k$-th component for the class~$c$; $P\left(D|k,\mathbf{x}\right)$,
the detection probability of the $k$-th component in the input $\mathbf{x}$.
$P\left(\overline{D}|k,\mathbf{x}\right)$ is the complementary probability,
that is, \emph{not} detecting the $k$-th component in $\mathbf{x}$.
An agreement $A$ is a path in the tree (see solid lines in \Figref{Probability-tree-diagram-CBC})
that depicts the positive influence of the $k$-th component on class
$c$ by either being detected ($D$) and required ($R$) or not detected
($\overline{D}$) and not required ($\overline{R}$). The output probability
$p_{c}\left(\mathbf{x}\right)=P\left(A|I,\mathbf{x},c\right)$ for
class $c$ is derived from the agreement $A$  using the following
expression:

\begin{equation}
\frac{\sum_{k}\left(P\left(R,I|k,c\right)P\left(D|k,\mathbf{x}\right)+P\left(\overline{R},I|k,c\right)P\left(\overline{D}|k,\mathbf{x}\right)\right)P\left(k\right)}{\sum_{k}\left(P\left(R,I|k,c\right)+P\left(\overline{R},I|k,c\right)\right)P\left(k\right)}.\label{eq:CBC_rasoning}
\end{equation}
The defined probabilities and components are learned by minimizing
the margin loss (maximizing the probability gap)
\begin{equation}
\min\left\{ \max_{c'\neq y}p_{c'}\left(\mathbf{x}\right)-p_{y}\left(\mathbf{x}\right)+\gamma,0\right\} ,\label{eq:margin_loss}
\end{equation}
with $\gamma\in\left[0,1\right]$ being the margin value, $y$ being
the correct class label of $\mathbf{x}$, and $c'$ being any class
label other than $y$. The model can be used without or with a feature
extractor (see \Figref{General-architecture}); that is, the distance
computation occurs in the learned latent space. An original CBC without
a feature extractor realizes an extension of traditional PBNs, overcoming
Limitation~(1) while posing new difficulties.

The architecture is difficult to train as it often converges to a
bad local minimum (see \secref{Experiments}), and the explanations
can be counterintuitive. To see this, note that $P\left(R,I|k,c\right)+P\left(\overline{R},I|k,c\right)+P\left(\overline{I}|k,c\right)=1$
for each $k$. Thus, one can scale the reasoning probabilities $P\left(R,I|k,c\right)$
and $P\left(\overline{R},I|k,c\right)$ in \Eqref{CBC_rasoning} by
any factor $\alpha>0$ as long as $P\left(R,I|k,c\right),P\left(\overline{R},I|k,c\right)\in[0,1]$
remains valid without changing the output probability $p_{c}\left(\mathbf{x}\right)$.
Assuming that $p_{c}\left(\mathbf{x}\right)=1$, this result can be
obtained from nearly zero reasoning probabilities, giving confident
predictions from infinitesimal reasoning evidence. This contradicts
the design principle of the original CBC approach, as $p_{c}\left(\mathbf{x}\right)=1$
should only be generated if the model is certain in its reasoning.
At the same time, this result implies that the optimal output ($p_{c}\left(\mathbf{x}\right)=1$)
is not unique with a wide range of flawed feasible solutions, thus
causing the model to converge to bad local minima. 

\paragraph{Our extension of the original CBC method.}

In CBC, both problems mentioned above are caused by the indefinite
reasoning probability $P\left(\overline{I}|k,c\right)$ together with
the component prior $P(k)$. These probabilities model the extent
to which a component is used in the classification process; hence,
they both serve the same purpose, as confirmed by fixing $P(k)$ to
be uniform in the original CBC\@. Removing $P\left(\overline{I}|k,c\right)$
from the model eliminates the problematic model's tolerance towards
scaling by a factor $\alpha$. Still, it causes missing support for
allowing components to remain irrelevant (to abstain), as explained
by \citet{Saralajew2019} in Figure 1. Similarly, allowing the prior
$P(k)$ to be trainable does not generalize to cover the property
of class-specific component priors.

We now present our modification to the original CBC to overcome the
difficulties. We propose to remove the importance variable $I$ and
substitute it with the \emph{trainable class-wise component} prior
$P\left(k\mid c\right)$, see \Figref{Probability-tree-diagram-CBC}.
The output probability $p_{c}\left(\mathbf{x}\right)=P\left(A\mid\mathbf{x},c\right)$,
using the agreement, becomes

\begin{align}
 & P\left(A\mid\mathbf{x},c\right)=\nonumber \\
 & \sum_{k}\left(P\left(R,D\mid x,c,k\right)+P\left(\overline{R},\overline{D}\mid\mathbf{x},c,k\right)\right)P\left(k\mid c\right)=\nonumber \\
 & \sum_{k}\left(P\left(R\mid c,k\right)P\left(D\mid\mathbf{x},k\right)+P\left(\overline{R}\mid c,k\right)P\left(\overline{D}\mid\mathbf{x},k\right)\right)P\left(k\mid c\right)\label{eq:long_notation_output_probability}
\end{align}
We introduce the following notations: 
\begin{itemize}
\item The requiredness possibility vector $\mathbf{r}_{c}\in\left[0,1\right]^{K}$
contains the probabilities $P\left(R\mid c,k\right)$ for all $k$. 
\item The detection possibility vector $\mathbf{d}\left(\mathbf{x}\right)\in\left[0,1\right]^{K}$
contains the probabilities $P\left(D\mid\mathbf{x},k\right)$ for
all $k$. 
\item The component prior probability vector $\mathbf{b}_{c}\in\left[0,1\right]^{K}$
contains the probabilities $P\left(k\mid c\right)$ for all $k$.
\end{itemize}
Note that $\sum_{k}b_{c,k}=\sum_{k}P\left(k\mid c\right)=1$, which
is not necessarily true for $\mathbf{d}$ and $\mathbf{r}_{c}$. Now,
\Eqref{long_notation_output_probability} can be written as
\begin{align}
p_{c}\left(\mathbf{x}\right) & =\left(\mathbf{r}_{c}\circ\mathbf{d}\left(\mathbf{x}\right)+\left(\mathbf{1}-\mathbf{r}_{c}\right)\circ\left(\mathbf{1}-\mathbf{d}\left(\mathbf{x}\right)\right)\right)^{\mathrm{T}}\mathbf{b}_{c},\label{eq:shorthand_output_prob}
\end{align}
where $\circ$ is the Hadamard product. The detection probability
can be any suitable function, like the following RBF:\footnote{The detection probability must be a similarity measure $\mathbb{R}^{n}\times\mathbb{R}^{n}\rightarrow\left[0,1\right]$
such that $\mathbf{x}=\mathbf{w}_{k}$ implies a similarity of $1.0$.}
\begin{equation}
P\left(D\mid\mathbf{x},k\right)=\exp\left(-\frac{d_{E}\left(\mathbf{x},\mathbf{w}_{k}\right)}{\sigma_{k}}\right),\label{eq:rbf-kernel-definition}
\end{equation}
where $d_{E}$ is the Euclidean distance, $\sigma_{k}$ is the (trainable)
component-dependent temperature, and $\mathbf{w}_{k}$ is the vector
representation of component $k$. Using \Eqref{shorthand_output_prob},
similarly to other deep PBNs, the architecture is trained by optimizing
the parameters of the components $\mathbf{w}_{k}$, the prior probabilities
$\mathbf{b}_{c}$, and the reasoning possibility vector $\mathbf{r}_{c}$.
For the optimization, the margin loss \Eqref{margin_loss} can be
used.

\paragraph{Learning the parameters in CBC models.}

When adopted without a feature extractor, learning a CBC model realizes
an extension of shallow PBNs using components instead of prototypes
(Limitation~(1) in \Secref{Motivation-and-Context}) and constitutes
an interpretable RBF network (fixes the interpretability issues mentioned
in \Secref{Review-of-Prototype-based}). Note that in the computation
of \Eqref{long_notation_output_probability}, the requiredness probabilities
$P\left(R\mid c,k\right)$ and the component prior probabilities $P\left(k\mid c\right)$
occur jointly and provide the \emph{reasoning probabilities} $P\left(R,k\mid c\right)=P\left(R\mid c,k\right)P\left(k\mid c\right)$.
This simplification makes the association to RBF networks more explicit
by rewriting \Eqref{long_notation_output_probability} as $p_{c}\left(\mathbf{x}\right)=\sum_{k}\alpha_{k}P\left(D\mid\mathbf{x},k\right)+\beta$,
where $\alpha_{k}=P\left(R,k\mid c\right)-P\left(\overline{R},k\mid c\right)$
is the weight and $\beta=\sum_{k}P\left(\overline{R},k\mid c\right)$
is the bias.

Moreover, the network is simplified during training, and only the
reasoning probabilities $P\left(R,k\mid c\right)$ are learned, leading
to fewer multiplications of trainable parameters (simpler gradient
computation graph). In practice, the trainable parameters $P\left(R,k\mid c\right)$
and $P\left(\overline{R},k\mid c\right)$ take the form of the vector
$\mathbf{v}_{c}\in\mathbb{R}^{2K}$ for each class, which is normalized
to achieve $\sum_{i}\mathrm{softmax}\left(\mathbf{v}_{c}\right)_{i}=1$
. Within $\mathbf{v}_{c}$, the first half of the parameters represent
the positive and the second half the negative reasoning probabilities.
The computation of $p_{c}\left(\mathbf{x}\right)$ becomes $\mathbf{v}_{c}^{\mathrm{T}}\left[\mathbf{d}\left(\mathbf{x}\right),\left(1-\mathbf{d}\left(\mathbf{x}\right)\right)\right]$,
where the detection and no detection vectors are concatenated into
one vector. Consequently, and again, the model realizes an RBF network
that uses negative reasoning. If we block negative reasoning by setting
the respective probabilities to zero, we obtain an RBF network with
class-wise weights constrained while solving the interpretability
issues from \Secref{Review-of-Prototype-based}.

\paragraph{The proven robustness of the CBC architecture.}

In this section, we derive the robustness lower bound. We analyze
the stability of the classification decision when \emph{no} feature
extractor is applied; with a feature extractor, the same stability
analysis applies in the latent space. Given a data point $\mathbf{x}\in\mathbb{R}^{n}$
with the target label $y$, the input is correctly classified if the
probability gap is positive: 
\begin{equation}
p_{y}\left(\mathbf{x}\right)-\max_{c'\neq y}p_{c'}\left(\mathbf{x}\right)>0.\label{eq:probability_gap}
\end{equation}
Robustness comes from deriving a non-trivial lower bound for the maximum
applicable perturbation $\boldsymbol{\epsilon}^{*}\in\mathbb{R}^{n}$
without having the predicted class label of $\mathbf{x}$ changed,
that is, 
\begin{equation}
p_{y}\left(\mathbf{x}+\boldsymbol{\epsilon}^{*}\right)-\max_{c'\neq y}p_{c'}\left(\mathbf{x}+\boldsymbol{\epsilon}^{*}\right)>0;\label{eq:probability_gap_perturb}
\end{equation}
the strength of the perturbation is given by $\left\Vert \boldsymbol{\epsilon}^{*}\right\Vert $.
\thmref{Robustness-with-individual-sigma} derives a lower bound of
$\left\Vert \boldsymbol{\epsilon}^{*}\right\Vert $ for detection
probability functions of the form \Eqref{rbf-kernel-definition} where
$d_{E}$ is \emph{any} distance function induced by the selected norm
$\left\Vert \cdot\right\Vert $. \thmref{lower-bound-standard-RBF}
extends this derivation to squared norms (e.\,g., Gaussian kernel)
so that the result can be applied to standard Gaussian RBF networks
using the established relation. 
\begin{thm}
\label{thm:Robustness-with-individual-sigma}The robustness of a correctly
classified sample $\mathbf{x}$ with class label $y$ is lower bounded
by

\begin{equation}
\left\Vert \boldsymbol{\epsilon}^{*}\right\Vert \geq\underbrace{\kappa\min_{c'\neq y}\left(\ln\left(-\frac{B_{c'}+\sqrt{B_{c'}^{2}-4A_{c'}C_{c'}}}{2A_{c'}}\right)\right)}_{=:\delta}>0,\label{eq:Lower_bound_perturbation_with_min}
\end{equation}
when $A_{c'}\neq0$, where
\begin{align*}
A_{c'} & =\left(\left(\mathbf{r}_{y}-\mathbf{1}\right)\circ\mathbf{b}_{y}-\mathbf{r}_{c'}\circ\mathbf{b}_{c'}\right)^{\mathrm{T}}\mathbf{d}\left(\mathbf{x}\right),\\
B_{c'} & =\left(\mathbf{1}-\mathbf{r}_{y}\right)^{\mathrm{T}}\mathbf{b}_{y}-\left(\mathbf{1}-\mathbf{r}_{c'}\right)^{\mathrm{T}}\mathbf{b}_{c'},\\
C_{c'} & =\left(\mathbf{r}_{y}\circ\mathbf{b}_{y}-\left(\mathbf{r}_{c'}-\mathbf{1}\right)\circ\mathbf{b}_{c'}\right)^{\mathrm{T}}\mathbf{d}\left(\mathbf{x}\right),
\end{align*}
and $\kappa=\sigma_{min}=\min_{k}\sigma_{k}$.
\end{thm}

\noindent All proofs can be found in \Appendixref{Derivation-of-the-lower-bounds}\@.
Additionally, it can be shown that $\delta$ in \Eqref{Lower_bound_perturbation_with_min}
is negative if the sample is incorrectly classified. Therefore, $\delta$
in \Eqref{Lower_bound_perturbation_with_min} can be used as a loss
function to optimize the model for stability. Of course, this loss
can be clipped at a threshold $\gamma>0$ so that the network optimizes
for robustness of at most $\gamma$.
\begin{thm}
\label{thm:lower-bound-standard-RBF}If we use the standard RBF kernel
(squared norm), then \Eqref{Lower_bound_perturbation_with_min} becomes
$\left\Vert \boldsymbol{\epsilon}^{*}\right\Vert \geq-\frac{\beta}{3}+\sqrt{\frac{\beta^{2}}{9}+\delta}>0$
with $\kappa=\frac{\sigma_{min}}{3}$ and $\beta=\max_{k}d\left(\mathbf{x},\mathbf{w}_{k}\right)$.
\end{thm}

\noindent Again, this result helps to construct a loss function that
maximizes robustness. For standard Gaussian kernel RBF networks with
class-wise weights $\mathbf{v}_{c}$ constrained to probability vectors,
the main part $\delta$ of the function is simplified to
\begin{equation}
\frac{\sigma_{min}}{6}\min_{c'\neq y}\ln\left(\frac{\mathbf{v}_{y}^{\mathrm{T}}\mathbf{d}\left(\mathbf{x}\right)}{\mathbf{v}_{c'}^{\mathrm{T}}\mathbf{d}\left(\mathbf{x}\right)}\right),\label{eq:robust_RBF_loss_pos_reasoning_only}
\end{equation}
a log-likelihood ratio loss \citep[e.\,g.,][]{Seo2003}.

\paragraph{The robustness with alternative distance functions. }

Similar to other shallow PBNs, CBCs can use alternative distance functions
such as the Mahalanobis distance or the tangent distance \citep[e.\,g.,][]{Haasdonk2002}
\begin{equation}
d_{T}\left(\mathbf{x},S\right)=\min_{\boldsymbol{\theta}\in\mathbb{R}^{r}}d_{E}\left(\mathbf{x},\mathbf{w}+\mathbf{W}\boldsymbol{\theta}\right),\label{eq:Tangent_Distance}
\end{equation}
where $S=\left\{ \mathbf{w}+\mathbf{W}\boldsymbol{\theta}\mid\boldsymbol{\theta}\in\mathbb{R}^{r}\right\} $
is a trainable $r$-dimensional affine subspace with $\mathbf{W}$
being a basis. By learning affine subspaces instead of points for
the components, the discriminative power of the architecture is significantly
improved \citep{Saralajew2020}. Moreover, if this distance is used
in a deep PBN, it realizes an extension of TesNet by learning disentangled
concepts (each basis vector in $\mathbf{W}$ is a basis concept) but
measures the distance with respect to $d_{T}$. See \appendixref{appendix_tangent_distance}
for further details about this distance. Next, \thmref{lower-bound-tangent-distance}
extends the lower bound derived in \thmref{Robustness-with-individual-sigma}
for the tangent distance.
\begin{thm}
\label{thm:lower-bound-tangent-distance}If we use the tangent distance
in \Eqref{rbf-kernel-definition}, \Eqref{Lower_bound_perturbation_with_min}
holds with $\kappa=\frac{1}{2}\sigma_{min}$ and $\left\Vert \cdot\right\Vert $
being the Euclidean norm.
\end{thm}

\noindent A similar result was proven for LVQ with the tangent distance
\citep{Saralajew2020}. 

\paragraph{Final remarks.}

Our proposed CBC resolves the original approach's drawbacks. Further,
the architecture can be derived from RBF networks by introducing interpretability
constraints and negative reasoning. The method can be used as a head
for deep PBNs or as a standalone for prototype-based classification
learning. In all cases, the interpretability of the learned weights
is guaranteed by the relation to the probability events. \appendixref{Further-theoretical-results}
presents further theoretical results.

\section{Experiments\label{sec:Experiments}}

In this section, we test our CBC and the presented theories: (1) We
analyze the accuracy and interpretability of our CBC and compare it
to PIPNet. (2) We compare shallow CBCs with other shallow models,
such as the original CBC\@. (3)~To demonstrate our theorems, we
analyze the adversarial robustness of shallow PBNs. Note that all
accuracy results are reported in percentage; we train each model five
times, and report the mean and standard deviation. \footnote{The source code is available at \url{https://github.com/si-cim/cbc-aaai-2025}.}

\paragraph{Interpretability and performance assessment: Comparison with PIPNet.}

\begin{table}
\begin{centering}
\small{%
\begin{tabular}{cccc}
 & CUB & CARS & PETS\tabularnewline
\midrule
\midrule 
PIPNet & $84.3\pm0.2$ & $88.2\pm0.5$ & $92.0\pm0.3$\tabularnewline
ProtoPool & $85.5\pm0.1$ & $88.9\pm0.1$ & $87.2^{*}\pm0.1$\tabularnewline
ProtoViT & $85.8\pm0.2$ & $92.4\pm0.1$ & $93.3^{*}\pm0.2$\tabularnewline
CBC & $\mathbf{87.8\pm0.1}$ & $\mathbf{93.0\pm0.0}$ & $\mathbf{93.9\pm0.1}$\tabularnewline
CBC pos.~reas. & $28.6\pm0.8$ & $25.3\pm2.3$ & $69.5\pm5.1$\tabularnewline
\end{tabular}}
\par\end{centering}
\caption{Test accuracy on different benchmark datasets. If available, we copied
the accuracy values from the respective papers. Otherwise, we computed
them (marked by an asterisk).\label{tab:Test-accuracy-PIPNet-vs-CBC}}
\end{table}
We evaluate the performance of CBC in comparison with PIPNet and the
state-of-the-art deep PBN ProtoPool and ProtoViT\footnote{It was not published when the submission draft for AAAI was written.}
(CaiT-XXS 24; best-performing backbone). Since CBC can work with any
backbone, we use PIPNet's ConvNeXt-tiny \citep{Liu2022} architecture,
the best-performing one from PIPNet. We extend PIPNet by only replacing
the final classification layer with a CBC head. This way, the components
become implicitly defined by the weights of the last convolutional
layer with softmax-normalized dot product as a similarity. For training,
we follow the pre-training protocol from PIPNet and extend the classification
step using our proposed margin loss \eqref{margin_loss} with $\gamma=0.025$.
We benchmark the methods using CUB, CARS \citep{Krause2013}, and
PETS \citep{Parkhi2012} datasets.

The test accuracy results of our model sets new benchmarks as shown
in \tabref{Test-accuracy-PIPNet-vs-CBC}. To analyze the reason for
this accuracy gain, we trained another PIPNet, replacing the ReLU
constraint on the classification weights with a softmax. By this,
we avoid the mentioned interpretability issues and obtain a CBC restricted
to positive reasoning only (CBC pos.~reas.). This model constantly
scores behind CBC with negative reasoning. Hence, the accuracy gain
can be attributed to the usefulness of negative reasoning. 

\begin{figure}
\begin{centering}
\includegraphics[viewport=8bp 7bp 429bp 244bp,clip,width=1\columnwidth]{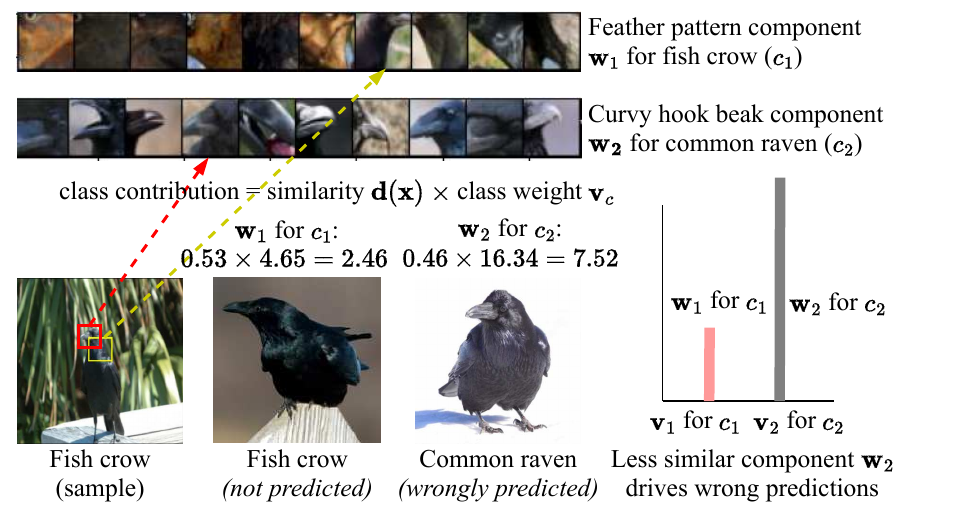}
\par\end{centering}
\caption{Fish crow gets incorrectly classified as common raven by PIPNet because
of the overemphasis of weights.\label{fig:Fish-crow-vs.common-raven}}
\end{figure}
\begin{figure}
\begin{centering}
\includegraphics[viewport=27bp 9bp 448bp 194bp,clip,width=1\columnwidth]{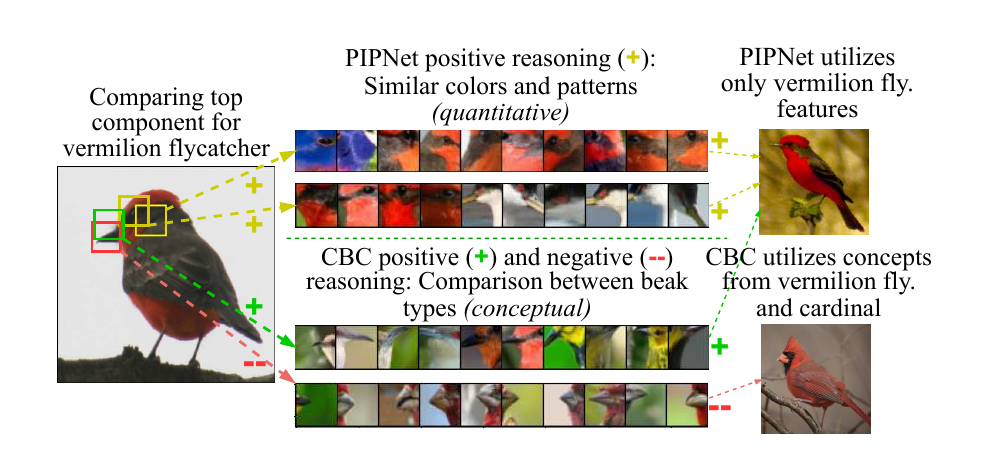}
\par\end{centering}
\caption{The comparative analysis of PIPNet and CBC for the vermilion flycatcher,
where negative reasoning is used.\label{fig:Vermillion-flycatcher-prediction}}
\end{figure}
To assess the interpretability, we use PIPNet's approach to determine
the top-10 component visualizations from the training dataset. \figref{Fish-crow-vs.common-raven}
shows an example that is wrongly classified by PIPNet due to the overemphasis
of specific weights. Ravens have curved hook-like beaks and regions
of larger feathers, whereas crows have streamlined beaks and small
feathers. The crow depicted in this figure is wrongly classified as
a raven because the most similar component $\mathbf{w}_{1}$ (feather),
which correctly indicates that it is a crow, is overshadowed by the
less similar component $\mathbf{w}_{2}$ (hook-like beak) that has
a higher weight. This example confirms our hypothesis from \secref{Review-of-Prototype-based}
that non-normalized weights hinder interpretability by preventing
the most relevant prototypes from influencing the prediction.

\figref{Vermillion-flycatcher-prediction} shows an example of positive
and negative reasoning to distinguish between two close bird species.
PIPNet uses positive reasoning to match based on regions with similar
colors or color contrasts, focusing less on contextual understanding.
CBC focuses on learning concepts like the pointed streamlined beak
irrespective of the bird species or color pattern patches. As this
component is similar to the beak of the depicted bird (vermillion
flycatcher), it contributes to the classification as a vermillion
flycatcher (positive reasoning). At the same time, CBC distinguishes
the vermillion flycatcher from a similar species in appearance, the
cardinal, by using negative reasoning with the absence of the cardinal's
broad beak. 

To quantitatively assess how different components are used across
different classes by learning class-specific component priors, we
computed the Jensen--Shannon divergence between the priors of each
pair of bird classes. The divergence depicts how the distributions
of the components' priors differ across classes. The following shows
this for the Black-footed Albatross compared to three other species:
Laysan Albatross $1.1$, Crested Auklet $5.1$, and Least Auklet $4.4$.
These results indicate a smaller divergence to the Laysan Albatross,
a close relative from the same family, and greater divergences to
more distantly related species. This demonstrates that our approach
generally shares components across similar classes while using different
components for others. Again, this result underlines the importance
of learning class-specific component priors. \appendixref{Interpretability-assessment:-PIPNet}
presents model training details, ResNet50 results, and more interpretability
results.

\paragraph{Comparison with shallow models.}

\begin{table}
\begin{centering}
\small{%
\begin{tabular}{cccc}
 & \multirow{1}{*}{Accuracy} & Emp.~Rob. & Cert.~Rob.\tabularnewline
\midrule
\midrule 
\multirow{1}{*}{GLVQ} & $80.5\pm0.6$ & $59.6\pm0.3$ & $\mathbf{32.3\pm0.3}$\tabularnewline
\multirow{1}{*}{RBF} & $\mathbf{92.2\pm0.1}$ & $61.9\pm0.9$ & $-$\tabularnewline
\multirow{1}{*}{original CBC} & $81.8\pm2.0$ & $\mathbf{62.0\pm1.0}$ & $-$\tabularnewline
\midrule
\multirow{1}{*}{CBC} & $87.4\pm0.3$ & $68.1\pm0.7$ & $0.2\pm0.1$\tabularnewline
\multirow{1}{*}{RBF-norm} & $77.3\pm0.2$ & $57.7\pm0.2$ & $0.7\pm0.0$\tabularnewline
\multirow{1}{*}{CBC TD} & $\mathbf{95.9\pm0.1}$ & $\mathbf{84.5\pm0.2}$ & $0.0\pm0.0$\tabularnewline
\multirow{1}{*}{RBF-norm TD} & $92.1\pm0.2$ & $77.8\pm0.4$ & $0.0\pm0.0$\tabularnewline
\multirow{1}{*}{Robust CBC} & $87.8\pm0.3$ & $62.8\pm0.3$ & $\mathbf{15.2\pm1.7}$\tabularnewline
\multirow{1}{*}{Robust CBC TD} & $91.9\pm0.3$ & $70.8\pm0.5$ & $1.6\pm0.2$\tabularnewline
\end{tabular}}
\par\end{centering}
\caption{Test, empirical robust, and certified robust accuracy of shallow PBNs.
The robust accuracy is computed for $\left\Vert \boldsymbol{\epsilon}^{*}\right\Vert =1$.
The top shows prior art, and the bottom shows our models. We put the
best accuracy for each category in bold.\label{tab:Shallow-results-short.}}
\end{table}
In this experiment, we compare CBC with its variants and other baseline
models. Namely, we compare with GLVQ \citep{Sato1996}, RBF networks,
and the original CBC\@. We also implement RBF networks with softmax
layer normalization (RBF-norm) and RBF networks with Tangent Distance
(RBF-norm TD); see \eqref{Tangent_Distance}. We evaluate CBC with
the Tangent Distance (CBC TD), with the robustness loss optimization
(Robust CBC; see \thmref{Robustness-with-individual-sigma}), and
with both the robustness loss and the Tangent Distance (Robust CBC
TD). All models are trained with the Euclidean distance unless the
use of the tangent distance is indicated. The RBF models are trained
by the cross-entropy loss, GLVQ by the GLVQ-loss function, and non-robust
CBC models by the margin loss (\eqref{margin_loss} with $\gamma=0.3$).
Each model was trained and evaluated on MNIST \citep{LeCun1998a}.
Each CBC and RBF can learn 20 components (or centroids) or two prototypes
per class (GLVQ). The CBC models are trained with \emph{two} reasoning
concepts per class (two vectors $\mathbf{r}_{c}$ and $\mathbf{b}_{c}$
per class), component-wise temperatures, and squared Euclidean distances.
The class output probability is given by the maximum over the class's
two reasoning concepts. By this, we ensure that, similar to GLVQ,
the models can learn two concepts (similar to prototypes) per class.

The results presented in \tabref{Shallow-results-short.} show that
CBC outperforms the original CBC in terms of classification accuracy
by over 5\,\%. By inspecting the learned components and probabilities,
we observe that the original CBC converges to a sub-optimal solution
by learning redundant components and not leveraging the advantage
of multiple reasoning concepts per class. Our CBC learns less repetitive
components and leverages the two reasoning concepts by learning class-specific
components for several classes if required. Additionally, the table
shows the advantage of using negative reasoning (cf.~CBC and RBF-norm).
While the class-wise softmax normalization in RBF-norm transforms
it to a CBC with positive reasoning only, it remains outperformed
by CBC with negative reasoning by 10\,\%. At the same time, both
RBF-norm and CBC remain behind the plain RBF approach, showing how
the interpretability constraints reduce the generalization. By using
more advanced distance measures such as the tangent distance, we observe
that the accuracy improves drastically while still being behind the
plain models if they use the tangent distance. See \appendixref{Original-CBC-vs.our-CBC}
for the complete set of results, including the comparison with more
shallow models, the component visualizations, training with non-squared
distances, and a shallow model with patch components, where the learned
reasoning distinguishes between writing styles of the numeral seven.

\paragraph{Robustness evaluation.}

We evaluate the adversarial robustness of the already trained models
from the shallow PBN experiments using the AutoAttack framework \citep{Croce2020}
with the recommended setting and maximum perturbation strength $1.0$,
see \tabref{Shallow-results-short.}. Additionally, using the result
from \thmref{lower-bound-standard-RBF}, we compute the certified
robustness by counting how many correctly classified samples have
a lower bound greater or equal to $1.0$. For GLVQ, we compute the
certified robustness by the hypothesis margin \citep{Saralajew2020}.
Note that the certified robustness cannot be calculated for RBF and
original CBC\@.

The results show that training a CBC with our robustified loss is
possible and yields non-trivial certified robustness. For instance,
the Robust CBC outperforms GLVQ, which is provably robust as well,
in terms of accuracy and empirical robustness. With respect to the
certified robustness, it is behind GLVQ, which can be attributed to
the repeated application of the triangle inequality in order to derive
the bound. Moreover, it should be noted that the certified robustness
of Robust CBC TD is significantly lower than that of Robust CBC\@.
This can be again attributed to the derived lower bound for the tangent
distance, where the triangle inequality is applied once more. Hence,
the stated bound in \thmref{lower-bound-tangent-distance} is less
tight compared to \thmref{Robustness-with-individual-sigma} and \thmref{lower-bound-standard-RBF}.
See Appx.~D.3 for the full results, including robustness curves
and evaluation of robustified RBF networks using \thmref{lower-bound-standard-RBF}.

\section{Discussion and Limitations}

While we refrain from claiming that our deep model is fully interpretable,
we believe it offers partial interpretability, providing valuable
insights into the classification process, especially in the final
layers. In contrast, the shallow version is inherently interpretable.

Compared to other deep PBNs, our model uses only a single loss term
and neither forces the components to be close to training samples
nor to be apart from each other. This is beneficial as it simplifies
the training procedure drastically since no regularization terms have
to be tuned. Even if we only use one loss term, our model converges
to valuable components. However, interpreting these components is
complex and requires expert knowledge. As a result, especially for
deep PBNs, the interpretation could be largely shaped by the user's
mental model, highlighting the importance of quantitative interpretation
assessment approaches---something that is still lacking in the field.
Additionally, by optimizing the single loss term, our model automatically
learns sparse component representations without the issue of the learned
representation being excessively sparse (see the additional PIPNet
experiments in \appendixref{Interpretability-assessment:-PIPNet}). 

During the deep model training, we observed that the CBC training
behavior can be sensitive to pre-training and initializations. Further,
training huge shallow models was challenging, especially when optimizing
the robust loss: The model did not leverage all components as they
often converged to the same point or failed to use all reasoning vectors
if multiple reasoning vectors per class were provided. Additionally,
training exponential functions (the detection probability) is sensitive
to the selection of suitable temperature values. When we kept them
trainable and individual per component, sometimes they became so small
that the components did not learn anything even if the components
had not converged to a suitable position in the data space. The same
happened when we tried to apply exponential functions on top of a
deep feature backbone, making it impossible to train such architectures
reliably. These insights provide a foundation for refining our approach
in future efforts.

\section{Conclusion and Outlook}

In this paper, we harmonize deep PBNs by showing a solid link to RBF
networks. We also show how these models are not interpretable and
only achieve partial interpretability in the best case. Inspired by
these findings, we derive an improved CBC architecture that uses negative
reasoning in a probabilistically sound way and ensures partial interpretability.
Empirically, we demonstrate that the proposed deep PBN outperforms
existing models on established benchmarks. Besides, the shallow version
of our CBC is interpretable and provably robust. The shallow CBC is
an attractive alternative to established models such as GLVQ as it
resolves known limitations like the use of class-specific prototypes. 

Open questions still exist and are left for future work: For example,
a modification that prevents components from converging to the same
point, along with the integration of spatial knowledge (to avoid global
max-pooling), could improve deep PBNs, a challenge that the original
CBC partially addressed. Moreover, in our evaluation, we focused on
the assessment of our approach using image datasets, which is currently
the commonly used benchmark domain for PBNs. However, future work
should investigate the application of our approach to other domains,
such as time series data. Moreover, to stabilize the training of the
detection probability, one should explore the strategies proposed
by \citet{GhiasiShirazi2019} or analyze the application of other
detection probability functions (note that our theoretical results
generalize to exponential functions with an arbitrary base). Finally,
it is unclear why all shallow models, including non-robustified ones,
exhibit good empirical robustness.

\bibliographystyle{style/AnonymousSubmission/LaTeX/aaai25}
\bibliography{main}

\newpage{}

\onecolumn
\setcounter{secnumdepth}{2} 

\appendix

\section{Derivation of the Tangent Distance and Extension to Restricted Versions\label{appendix:appendix_tangent_distance}}

The tangent distance is a transformation-invariant measure. Instead
of learning an individual prototype, it learns an affine subspace
to model the data manifold of a given class \citep{Haasdonk2002,Hastie1995}.
Its effectiveness was demonstrated multiple times. Given an affine
subspace that models the data and an input sample, the tangent distance
is defined as the minimal Euclidean distance between the affine subspace
and the input sample:

\[
d_{T}\left(\mathbf{x},S\right)=\min_{\boldsymbol{\theta}\in\mathbb{R}^{r}}d_{E}\left(\mathbf{x},\mathbf{w}+\mathbf{W}\boldsymbol{\theta}\right),
\]
where $S=\left\{ \mathbf{w}+\mathbf{W}\boldsymbol{\theta}\mid\boldsymbol{\theta}\in\mathbb{R}^{r}\right\} $
is an $r$-dimensional affine subspace with $\mathbf{W}$ being an
orthonormal basis (i.\,e., $\mathbf{W}^{\mathrm{T}}\mathbf{W}=\mathbf{I}$).
It can be shown that the minimizer $\mathbf{w}^{*}\left(\mathbf{x}\right)$
is given by 
\begin{equation}
\mathbf{w}^{*}\left(\mathbf{x}\right)=\mathbf{w}+\mathbf{W}\mathbf{W}^{\mathrm{T}}\left(\mathbf{x}-\mathbf{w}\right),\label{eq:TD_minimizer}
\end{equation}
which is the best approximating element. Using this result, the tangent
distance becomes
\begin{align*}
d_{T}\left(\mathbf{x},S\right) & =\left\Vert \mathbf{x}-\mathbf{w}-\mathbf{W}\mathbf{W}^{\mathrm{T}}\left(\mathbf{x}-\mathbf{w}\right)\right\Vert _{E}\\
 & =\left\Vert \left(\mathbf{I}-\mathbf{W}\mathbf{W}^{\mathrm{T}}\right)\left(\mathbf{x}-\mathbf{w}\right)\right\Vert _{E}.
\end{align*}
Note that $\mathbf{I}-\mathbf{W}\mathbf{W}^{\mathrm{T}}$ is an orthogonal
projector and, hence, is idempotent, which implies
\[
d_{T}\left(\mathbf{x},S\right)=\sqrt{\left(\mathbf{x}-\mathbf{w}\right)^{\mathrm{T}}\left(\mathbf{I}-\mathbf{W}\mathbf{W}^{\mathrm{T}}\right)\left(\mathbf{x}-\mathbf{w}\right)}.
\]
This equation can be used as a dissimilarity measure in classification
learning frameworks where the affine subspace is learned from data
\citep{Saralajew2016}. Moreover, this equation can be efficiently
implemented and even generalized to sliding operations (similar to
a convolution that uses the dot product) on parallel computing hardware.
If the measure is used to learn the affine subspaces, it is important
that the basis matrix is orthonormalized after each update step or
that a proper encoding is applied. For instance, the former can be
achieved by a polar decomposition via SVD and the latter by coding
the matrices as Householder matrices \citep{Mathiasen2020}. After
learning the affine subspaces, $\mathbf{W}$ captures the invariant
class dimensions, which are dimensions that are invariant with respect
to class discrimination. Moreover, the vector $\mathbf{w}$ represents
a data point similar to an ordinary prototype, a point that represents
the surrounding data as well as possible.

There are also extensions of this dissimilarity measure that constrain
the affine subspace. For example, one can define a threshold $\gamma>0$
and modify \eqref{Tangent_Distance} to
\[
d_{CT}\left(\mathbf{x},S\right)=\min_{\boldsymbol{\theta}\in\mathbb{R}^{r},\left\Vert \boldsymbol{\theta}\right\Vert _{E}\leq\gamma}d_{E}\left(\mathbf{x},\mathbf{w}+\mathbf{W}\boldsymbol{\theta}\right),
\]
which constrains the $r$-dimensional affine subspace to an $r$-dimensional
hyperball. The solution for this distance is 
\[
d_{CT}\left(\mathbf{x},S\right)=\sqrt{d_{E}^{2}\left(\mathbf{x},\mathbf{w}^{*}\left(\mathbf{x}\right)\right)+\left(\max\left\{ 0,d_{E}\left(\mathbf{w},\mathbf{w}^{*}\left(\mathbf{x}\right)\right)-\gamma\right\} \right)^{2}}.
\]
Like before, this measure can be efficiently implemented so that it
is possible to learn these hyperballs from data, which are like affine
subspaces that know the neighborhood they are approximating.

\section{Derivation of the Robust Lower Bounds\label{appendix:Derivation-of-the-lower-bounds}}

In the following, we prove the presented theorems. For this, we prove
a lemma that simplifies assumptions such as a class-independent temperature.
Then, using the lemma, we prove \thmref{Robustness-with-individual-sigma}.
Later on, based on \thmref{Robustness-with-individual-sigma} and
the proven lemma, we prove \thmref{lower-bound-standard-RBF} and~\ref{thm:lower-bound-tangent-distance}.

\subsection{Robustness lower bound for component-independent temperature and
a specific incorrect class}
\begin{lem}
\label{lem:Lemma}The robustness of a correctly classified sample
$\mathbf{x}$ with class label $y$ with respect to another class
$c'$ and temperature $\sigma_{k}=\sigma$ for all components in the
detection probability \Eqref{rbf-kernel-definition}, where the distance
is any distance $d\left(\cdot,\cdot\right)$ induced by a norm $\left\Vert \cdot\right\Vert $,
is lower bounded by

\begin{equation}
\left\Vert \boldsymbol{\epsilon}^{*}\right\Vert \geq\sigma\ln\left(-\frac{B+\sqrt{B^{2}-4AC}}{2A}\right)>0,\label{eq:Lower_bound_with_constant_sigma}
\end{equation}
when $A\neq0$, where
\begin{align*}
A & =\left(\left(\mathbf{r}_{y}-\mathbf{1}\right)\circ\mathbf{b}_{y}-\mathbf{r}_{c'}\circ\mathbf{b}_{c'}\right)^{\mathrm{T}}\mathbf{d}\left(\mathbf{x}\right),\\
B & =\left(\mathbf{1}-\mathbf{r}_{y}\right)^{\mathrm{T}}\mathbf{b}_{y}-\left(\mathbf{1}-\mathbf{r}_{c'}\right)^{\mathrm{T}}\mathbf{b}_{c'},\\
C & =\left(\mathbf{r}_{y}\circ\mathbf{b}_{y}-\left(\mathbf{r}_{c'}-\mathbf{1}\right)\circ\mathbf{b}_{c'}\right)^{\mathrm{T}}\mathbf{d}\left(\mathbf{x}\right).
\end{align*}

\end{lem}

\begin{proof}
To derive this bound, we perform the following steps:
\begin{enumerate}
\item We lower bound the probability gap 
\begin{equation}
p_{y}\left(\mathbf{x}+\boldsymbol{\epsilon}\right)-p_{c'}\left(\mathbf{x}+\boldsymbol{\epsilon}\right)>0\label{eq:Simplified_probability_gap}
\end{equation}
for an arbitrary $\boldsymbol{\epsilon}\in\mathbb{R}^{n}$ by using
the triangle inequality. 
\item We show that the derived lower bound for the probability gap is (strictly)
monotonic decreasing with respect to increasing $\left\Vert \boldsymbol{\epsilon}\right\Vert $.
\item We show that the derived lower bound has one root. This root is a
lower bound for the maximum perturbation $\left\Vert \boldsymbol{\epsilon}^{*}\right\Vert $
as the increase of $\left\Vert \boldsymbol{\epsilon}\right\Vert $
beyond this value results in a negative lower bound of the probability
gap (because of the monotonic decreasing behavior) and, hence, potential
misclassification.
\end{enumerate}

\paragraph*{Lower bound the probability gap.}

We use the triangle inequality and conclude that
\begin{equation}
\exp\left(-\frac{d\left(\mathbf{x}+\boldsymbol{\epsilon},\mathbf{w}_{k}\right)}{\sigma}\right)\geq\exp\left(-\frac{d\left(\mathbf{x},\mathbf{w}_{k}\right)}{\sigma}\right)\underbrace{\exp\left(-\frac{\left\Vert \boldsymbol{\epsilon}\right\Vert }{\sigma}\right)}_{=:\exp(-z)},\label{eq:lower_bound_detection_prob}
\end{equation}
where $d\left(\mathbf{x}+\boldsymbol{\epsilon},\mathbf{w}_{k}\right)\leq d\left(\mathbf{x}+\boldsymbol{\epsilon},\mathbf{x}\right)+d\left(\mathbf{x},\mathbf{w}_{k}\right)=\left\Vert \boldsymbol{\epsilon}\right\Vert +d\left(\mathbf{x},\mathbf{w}_{k}\right)$
and $z=\frac{\left\Vert \boldsymbol{\epsilon}\right\Vert }{\sigma}$
has been used. Similarly, we conclude that
\begin{equation}
\exp\left(-\frac{d\left(\mathbf{x}+\boldsymbol{\epsilon},\mathbf{w}_{k}\right)}{\sigma}\right)\leq\exp\left(-\frac{d\left(\mathbf{x},\mathbf{w}_{k}\right)}{\sigma}\right)\underbrace{\exp\left(\frac{\left\Vert \boldsymbol{\epsilon}\right\Vert }{\sigma}\right)}_{=\exp(z)}\label{eq:upper_bound_detection_prob}
\end{equation}
 by using $d\left(\mathbf{x}+\boldsymbol{\epsilon},\mathbf{w}_{k}\right)\geq-d\left(\mathbf{x}+\boldsymbol{\epsilon},\mathbf{x}\right)+d\left(\mathbf{x},\mathbf{w}_{k}\right)=-\left\Vert \boldsymbol{\epsilon}\right\Vert +d\left(\mathbf{x},\mathbf{w}_{k}\right)$.

Now, we use these results and lower bound the probability gap \Eqref{Simplified_probability_gap}.
In the first step, we apply the lower and upper bound for the disturbed
detection probability, see \Eqref{lower_bound_detection_prob} and
\Eqref{upper_bound_detection_prob}, respectively, to lower bound
the output probability \eqref{shorthand_output_prob} for the correct
class:
\begin{align}
p_{y}\left(\mathbf{x}+\boldsymbol{\epsilon}\right) & =\left(\mathbf{r}_{y}\circ\mathbf{d}\left(\mathbf{x}+\boldsymbol{\epsilon}\right)+\left(\mathbf{1}-\mathbf{r}_{y}\right)\circ\left(\mathbf{1}-\mathbf{d}\left(\mathbf{x}+\boldsymbol{\epsilon}\right)\right)\right)^{\mathrm{T}}\mathbf{b}_{y}\nonumber \\
 & \geq\left(\mathbf{r}_{y}\circ\mathbf{d}\left(\mathbf{x}\right)\exp(-z)+\left(\mathbf{1}-\mathbf{r}_{y}\right)\circ\left(\mathbf{1}-\mathbf{d}\left(\mathbf{x}\right)\exp(z)\right)\right)^{\mathrm{T}}\mathbf{b}_{y}\label{eq:lower_bound_correct_class}\\
 & =\left(\mathbf{r}_{y}\circ\mathbf{d}\left(\mathbf{x}\right)\exp(-z)+\mathbf{1}-\mathbf{d}\left(\mathbf{x}\right)\exp(z)-\mathbf{r}_{y}+\mathbf{r}_{y}\circ\mathbf{d}\left(\mathbf{x}\right)\exp(z)\right)^{\mathrm{T}}\mathbf{b}_{y}\nonumber \\
 & =\left(\mathbf{r}_{y}\circ\mathbf{b}_{y}\right)^{\mathrm{T}}\mathbf{d}\left(\mathbf{x}\right)\exp(-z)+\left(\mathbf{1}-\mathbf{r}_{y}\right)^{\mathrm{T}}\mathbf{b}_{y}+\left(\left(\mathbf{r}_{y}-\mathbf{1}\right)\circ\mathbf{b}_{y}\right)^{\mathrm{T}}\mathbf{d}\left(\mathbf{x}\right)\exp(z).\nonumber 
\end{align}
Note that this bound holds with equality (becomes the undisturbed
probability gap) if $z=0$. Next, we upper bound the output probability
for the incorrect class:
\begin{align}
p_{c'}\left(\mathbf{x}+\boldsymbol{\epsilon}\right) & =\left(\mathbf{r}_{c'}\circ\mathbf{d}\left(\mathbf{x}+\boldsymbol{\epsilon}\right)+\left(\mathbf{1}-\mathbf{r}_{c'}\right)\circ\left(\mathbf{1}-\mathbf{d}\left(\mathbf{x}+\boldsymbol{\epsilon}\right)\right)\right)^{\mathrm{T}}\mathbf{b}_{c'}\nonumber \\
 & \leq\left(\mathbf{r}_{c'}\circ\mathbf{d}\left(\mathbf{x}\right)\exp(z)+\left(\mathbf{1}-\mathbf{r}_{c'}\right)\circ\left(\mathbf{1}-\mathbf{d}\left(\mathbf{x}\right)\exp(-z)\right)\right)^{\mathrm{T}}\mathbf{b}_{c'}\label{eq:upper_bound_incorrect_class}\\
 & =\left(\mathbf{r}_{c'}\circ\mathbf{b}_{c'}\right)^{\mathrm{T}}\mathbf{d}\left(\mathbf{x}\right)\exp(z)+\left(\mathbf{1}-\mathbf{r}_{c'}\right)^{\mathrm{T}}\mathbf{b}_{c'}+\left(\left(\mathbf{r}_{c'}-\mathbf{1}\right)\circ\mathbf{b}_{c'}\right)^{\mathrm{T}}\mathbf{d}\left(\mathbf{x}\right)\exp(-z).\nonumber 
\end{align}
Again, note that this bound holds with equality if $z=0$. Combining
the two results yields
\begin{align*}
p_{y}\left(\mathbf{x}+\boldsymbol{\epsilon}\right)-p_{c'}\left(\mathbf{x}+\boldsymbol{\epsilon}\right) & \geq C\exp(-z)+A\exp(z)+B=:f(z),
\end{align*}
which holds with equality if $z=0$ and whereby
\begin{align*}
A & =\left(\left(\mathbf{r}_{y}-\mathbf{1}\right)\circ\mathbf{b}_{y}-\mathbf{r}_{c'}\circ\mathbf{b}_{c'}\right)^{\mathrm{T}}\mathbf{d}\left(\mathbf{x}\right),\\
B & =\left(\mathbf{1}-\mathbf{r}_{y}\right)^{\mathrm{T}}\mathbf{b}_{y}-\left(\mathbf{1}-\mathbf{r}_{c'}\right)^{\mathrm{T}}\mathbf{b}_{c'},\\
C & =\left(\mathbf{r}_{y}\circ\mathbf{b}_{y}-\left(\mathbf{r}_{c'}-\mathbf{1}\right)\circ\mathbf{b}_{c'}\right)^{\mathrm{T}}\mathbf{d}\left(\mathbf{x}\right).
\end{align*}

\paragraph*{The lower bound is monotonic decreasing.}

Next, we show that the function $f$ is monotonic decreasing. Assume
$z_{1}<z_{2}$ and show that $f(z_{1})\geq f(z_{2})$:
\begin{align*}
C\exp(-z_{1})+A\exp(z_{1}) & \geq C\exp(-z_{2})+A\exp(z_{2}),\\
C\underbrace{\left(\exp(-z_{1})-\exp(-z_{2})\right)}_{>0}+A\underbrace{\left(\exp(z_{1})-\exp(z_{2})\right)}_{<0} & \geq0.
\end{align*}
Considering the coefficient $A$, we can conclude that it is negative:
\[
A=\left(\underbrace{\left(\mathbf{r}_{y}-\mathbf{1}\right)\circ\mathbf{b}_{y}}_{\in\left[-1,0\right]^{K}}-\underbrace{\mathbf{r}_{c'}\circ\mathbf{b}_{c'}}_{\in\left[0,1\right]^{K}}\right)^{\mathrm{T}}\underbrace{\mathbf{d}\left(\mathbf{x}\right)}_{\in\left[0,1\right]^{K}}\leq0.
\]
Similarly, we can conclude that the coefficient $B\in[-1,1]$ and
$C\geq0$. Consequently, this implies that the function is monotonic
decreasing with respect to $z$ and even strictly monotonic decreasing
if $C$ or $A$ is unequal zero.

\paragraph*{Computing the root.}

Now, we want to compute a solution $z_{0}\in\mathbb{R}^{+}$ such
that $f(z_{0})=0$, which means finding $z_{0}$ for which the lower
bound of the probability gap is zero. This also means that for points
$z<z_{0}$ (smaller perturbations) all perturbations will not lead
to a change of the class assignment as $f(z)>0$ (concluded from the
decreasing monotonic behavior). Similarly, for points above the root
$z>z_{0}$ the lower bound will be negative ($f(z)<0$) so that it
cannot be guaranteed that there is no misclassification under a perturbation
of strength $z$.

To compute the root, we solve the equation 
\begin{equation}
C\exp(-z)+A\exp(z)+B=0\label{eq:quadratic_equation}
\end{equation}
 by multiplying with $\exp(z)$ and substituting $\exp(z)$ with $\tilde{z}$.
This leads to
\[
C+A\tilde{z}^{2}+B\tilde{z}=0.
\]
The solution for this quadratic equation is
\begin{equation}
\tilde{z}_{1,2}=-\frac{B}{2A}\pm\frac{1}{2\left|A\right|}\sqrt{B^{2}-4AC}.\label{eq:quadratic_solution}
\end{equation}
Note that this proof applies only to $A\neq0$; we discuss the case
$A=0$ after this proof. Considering that the coefficient $A$ is
negative, we can conclude that 
\[
\tilde{z}_{1,2}=\frac{B\pm\sqrt{B^{2}-4AC}}{2\left|A\right|}.
\]
Moreover, $B\in[-1,1]$ and $C\geq0$ implies that 
\begin{equation}
\left|B\right|\leq\sqrt{B^{2}-4AC},\label{eq:numerator_bound}
\end{equation}
and, further, that 
\[
\tilde{z}_{1}=\frac{B+\sqrt{B^{2}-4AC}}{2\left|A\right|}\geq0
\]
is the potential solution because $\tilde{z}$ must be positive to
be a valid solution for $\exp(z)$. Additionally, we have to show
that $\tilde{z}_{1}>1$ because $z$ must be positive. For this, we
first show that 
\begin{equation}
\sqrt{B^{2}-4AC}+\left(2\left|A\right|-B\right)>0.\label{eq:binomial_part_is_positive}
\end{equation}
This can be shown through the following steps using the result from
\Eqref{numerator_bound}:
\begin{align*}
\frac{\overbrace{B-\sqrt{B^{2}-4AC}}^{\leq0}}{2\left|A\right|} & <1,\\
-B+\sqrt{B^{2}-4AC} & >-2\left|A\right|,\\
\sqrt{B^{2}-4AC}+\left(2\left|A\right|-B\right) & >0.
\end{align*}
Now we show that $\tilde{z}_{1}>1$ by using the fact that $A+B+C=p_{y}(\mathbf{x})-p_{c'}(\mathbf{x})>0$:
\begin{align}
0 & <\left|A\right|\left(p_{y}(\mathbf{x})-p_{c'}(\mathbf{x})\right)\label{eq:Change_negative_prob_gap}\\
 & =\left|A\right|\left(C+B+A\right)\nonumber \\
 & =-AC+\left|A\right|B-\left|A\right|^{2}.\nonumber 
\end{align}
We multiply by $4$ and add the term $B^{2}$ on both sides:
\[
0<\left(B^{2}-4AC\right)-\left(4\left|A\right|^{2}-4\left|A\right|B+B^{2}\right).
\]
Finally, we recognize the structure of $\left(x-y\right)\left(x+y\right)=x^{2}-y^{2}$
and cancel $x+y$ by using \Eqref{binomial_part_is_positive}, which
completes the proof:
\begin{align*}
0 & <\left(B^{2}-4AC\right)-\left(4\left|A\right|^{2}-4\left|A\right|B+B^{2}\right),\\
0 & <\left(\sqrt{B^{2}-4AC}+\left(2\left|A\right|-B\right)\right)\left(\sqrt{B^{2}-4AC}-\left(2\left|A\right|-B\right)\right),\\
0 & <\sqrt{B^{2}-4AC}-\left(2\left|A\right|-B\right),\\
1 & <\frac{B+\sqrt{B^{2}-4AC}}{2\left|A\right|}=\tilde{z}_{1}.
\end{align*}
In summary, the solution for the \emph{lower bound of disturbed probability
gap} is
\begin{equation}
\left\Vert \boldsymbol{\epsilon}_{0}\right\Vert =\sigma\ln\left(-\frac{B+\sqrt{B^{2}-4AC}}{2A}\right)>0.\label{eq:epsilon_1_solution}
\end{equation}
Because this solution was computed for the lower bound of the probability
gap, the robustness $\left\Vert \boldsymbol{\epsilon}^{*}\right\Vert $
of a correctly classified sample $\mathbf{x}$ with class label $y$
with respect to another class is lower bounded by
\begin{equation}
\left\Vert \boldsymbol{\epsilon}^{*}\right\Vert \geq\sigma\ln\left(-\frac{B+\sqrt{B^{2}-4AC}}{2A}\right)>0.\label{eq:Lower_bound_perturbation}
\end{equation}

\end{proof}
The previous lemma proves the robustness bound for when $A\neq0$.
This is not a restriction as an even simpler result can be obtained
for the special case: If $A=0$, \Eqref{quadratic_equation} simplifies
to 
\[
\exp(-z)=-\frac{B}{C}.
\]
Taking into account that $A=0$ and that $\mathbf{d}\left(\mathbf{x}\right)\neq\mathbf{0}$
for all $\mathbf{x}\in\mathbb{R}^{n}$, this implies that
\[
-\mathbf{r}_{c'}\circ\mathbf{b}_{c'}=\left(\mathbf{1}-\mathbf{r}_{y}\right)\circ\mathbf{b}_{y}.
\]
Using this result, $C$ simplifies to 
\[
C=\left(\mathbf{b}_{y}+\mathbf{b}_{c'}\right)^{\mathrm{T}}\mathbf{d}\left(\mathbf{x}\right)\geq0
\]
and $B$ to
\[
B=-\mathbf{1}^{\mathrm{T}}\mathbf{b}_{c'}=-1,
\]
since $\mathbf{b}_{c'}$ is a probability vector. Moreover, $A+B+C>0$
implies that 
\[
-\frac{B}{C}<1.
\]
Further, by substituting $B$ and $C$ we get
\[
-\frac{B}{C}=\frac{1}{\left(\mathbf{b}_{y}+\mathbf{b}_{c'}\right)^{\mathrm{T}}\mathbf{d}\left(\mathbf{x}\right)}>0
\]
so that we conclude 
\[
-\frac{B}{C}\in\left(0,1\right)
\]
Consequently, the solution $z=-\ln\left(-\frac{B}{C}\right)$ is positive
and valid and we get 
\[
\left\Vert \boldsymbol{\epsilon}^{*}\right\Vert \geq\sigma\ln\left(\mathbf{b}_{y}^{\mathrm{T}}\mathbf{d}\left(\mathbf{x}\right)+\mathbf{b}_{c'}^{T}\mathbf{d}\left(\mathbf{x}\right)\right)>0.
\]
It must be noted that we assumed $C=\left(\mathbf{b}_{y}+\mathbf{b}_{c'}\right)^{\mathrm{T}}\mathbf{d}\left(\mathbf{x}\right)\neq0$,
which is valid since $\mathbf{d}\left(\mathbf{x}\right)\neq\mathbf{0}$
and $\mathbf{b}_{y}+\mathbf{b}_{c'}\neq\mathbf{0}$. In practice,
when we used the bound of \lemref{Lemma} for robustness evaluations
or model training, we never observed the special case of $A=0$. Hence,
we will not consider this special case for the following proofs. But
we emphasize that all results can be extended for this special case
so that focusing on $A\neq0$ is not a restriction.

In case of an incorrect classification, \Eqref{Change_negative_prob_gap}
changes to be less than zero as $A+B+C<0$. Then, this leads to $\tilde{z}_{1}<1$
so that the expression \Eqref{Lower_bound_perturbation} becomes negative:
\[
\sigma\ln\left(-\frac{B+\sqrt{B^{2}-4AC}}{2A}\right)<0.
\]
Hence, the sign of this expression follows the sign of the probability
gap. Consequently, this expression can be used to formulate a loss
that optimizes for robustness and correct classifications.

\subsection{Proof of \thmref{Robustness-with-individual-sigma}\label{appendix:Proof-of-theorem-initial}}

We now prove \thmref{Robustness-with-individual-sigma} by using \Lemref{Lemma}.
For completeness we restate the theorem:
\begin{thm*}
The robustness of a correctly classified sample $\mathbf{x}$ with
class label $y$ is lower bounded by

\[
\left\Vert \boldsymbol{\epsilon}^{*}\right\Vert \geq\underbrace{\kappa\min_{c'\neq y}\left(\ln\left(-\frac{B_{c'}+\sqrt{B_{c'}^{2}-4A_{c'}C_{c'}}}{2A_{c'}}\right)\right)}_{=:\delta}>0,
\]
when $A_{c'}\neq0$, where
\begin{align*}
A_{c'} & =\left(\left(\mathbf{r}_{y}-\mathbf{1}\right)\circ\mathbf{b}_{y}-\mathbf{r}_{c'}\circ\mathbf{b}_{c'}\right)^{\mathrm{T}}\mathbf{d}\left(\mathbf{x}\right),\\
B_{c'} & =\left(\mathbf{1}-\mathbf{r}_{y}\right)^{\mathrm{T}}\mathbf{b}_{y}-\left(\mathbf{1}-\mathbf{r}_{c'}\right)^{\mathrm{T}}\mathbf{b}_{c'},\\
C_{c'} & =\left(\mathbf{r}_{y}\circ\mathbf{b}_{y}-\left(\mathbf{r}_{c'}-\mathbf{1}\right)\circ\mathbf{b}_{c'}\right)^{\mathrm{T}}\mathbf{d}\left(\mathbf{x}\right),
\end{align*}
and $\kappa=\sigma_{min}=\min_{k}\sigma_{k}$.

\end{thm*}
\begin{proof}
The proof follows the technique used to prove \Lemref{Lemma} with
the following changes: To account for a component-wise $\sigma$,
we lower (upper) bound \eqref{lower_bound_detection_prob} and \eqref{upper_bound_detection_prob}
again, respectively,
\begin{eqnarray}
\exp\left(-\frac{d\left(\mathbf{x}+\boldsymbol{\epsilon},\mathbf{w}_{k}\right)}{\sigma_{k}}\right) & \geq & \exp\left(-\frac{d\left(\mathbf{x},\mathbf{w}_{k}\right)}{\sigma_{k}}\right)\exp\left(-\frac{\left\Vert \boldsymbol{\epsilon}\right\Vert }{\sigma_{k}}\right)\label{eq:lower_bound_detection_prob_k}\\
 & \ge & \exp\left(-\frac{d\left(\mathbf{x},\mathbf{w}_{k}\right)}{\sigma_{k}}\right)\exp\left(-\frac{\left\Vert \boldsymbol{\epsilon}\right\Vert }{\text{\ensuremath{\sigma_{min}}}}\right),
\end{eqnarray}
\begin{eqnarray}
\exp\left(-\frac{d\left(\mathbf{x}+\boldsymbol{\epsilon},\mathbf{w}_{k}\right)}{\sigma_{k}}\right) & \leq & \exp\left(-\frac{d\left(\mathbf{x},\mathbf{w}_{k}\right)}{\sigma_{k}}\right)\exp\left(\frac{\left\Vert \boldsymbol{\epsilon}\right\Vert }{\sigma_{k}}\right)\label{eq:upper_bound_detection_prob_k}\\
 & \le & \exp\left(-\frac{d\left(\mathbf{x},\mathbf{w}_{k}\right)}{\sigma_{k}}\right)\exp\left(\frac{\left\Vert \boldsymbol{\epsilon}\right\Vert }{\sigma_{min}}\right),
\end{eqnarray}
where $\sigma_{k}$ is the component-wise temperature, and $\sigma_{min}=\min\left\{ \sigma_{1},\ldots,\sigma_{K}\right\} $.
Consequently, the lower bound for the correct class becomes
\begin{align}
p_{y}\left(\mathbf{x}+\boldsymbol{\epsilon}\right) & \geq\left(\mathbf{r}_{y}\circ\mathbf{b}_{y}\right)^{\mathrm{T}}\mathbf{d}\left(\mathbf{x}\right)\exp(-z_{min})+\left(\mathbf{1}-\mathbf{r}_{y}\right)^{\mathrm{T}}\mathbf{b}_{y}+\left(\left(\mathbf{r}_{y}-\mathbf{1}\right)\circ\mathbf{b}_{y}\right)^{\mathrm{T}}\mathbf{d}\left(\mathbf{x}\right)\exp\left(z_{min}\right),\label{eq:lower_bound_correct_class-1}
\end{align}
and the upper bound of the output probability for an incorrect class
becomes
\begin{align}
p_{c'}\left(\mathbf{x}+\boldsymbol{\epsilon}\right) & \leq\left(\mathbf{r}_{c'}\circ\mathbf{b}_{c'}\right)^{\mathrm{T}}\mathbf{d}\left(\mathbf{x}\right)\exp(z_{min})+\left(\mathbf{1}-\mathbf{r}_{c'}\right)^{\mathrm{T}}\mathbf{b}_{c'}+\left(\left(\mathbf{r}_{c'}-\mathbf{1}\right)\circ\mathbf{b}_{c'}\right)^{\mathrm{T}}\mathbf{d}\left(\mathbf{x}\right)\exp\left(-z_{min}\right),\label{eq:upper_bound_incorrect_class_k}
\end{align}
where $z_{min}=\frac{\left\Vert \boldsymbol{\epsilon}\right\Vert }{\sigma_{min}}$. 

Next, we assume that $c'$ is any class label of an incorrect class.
Following the steps from \lemref{Lemma}, we get the solution 
\begin{equation}
\left\Vert \boldsymbol{\epsilon}^{*}\right\Vert \geq\sigma_{min}\ln\left(-\frac{B+\sqrt{B^{2}-4AC}}{2A}\right)>0.\label{eq:Lower_bound_perturbation_min}
\end{equation}
Since we search for the smallest perturbation $\left\Vert \boldsymbol{\epsilon}\right\Vert $
that changes the prediction, we have to compute the bound for each
class $c'\neq y$ and have to pick the minimum, which completes the
proof.
\end{proof}
In case multiple reasoning vectors per class are used, $\mathbf{r}_{c}$
and $\mathbf{b}_{c}$ become matrices containing the reasoning probability
vectors $\mathbf{r}_{c,i}$ and $\mathbf{b}_{c,i}$ where $i$ is
the index. In this case, the classifier takes the maximum probability
per class
\[
p_{c}\left(\mathbf{x}\right)=\max_{i}p_{c,i}\left(\mathbf{x}\right),
\]
and in \eqref{Lower_bound_perturbation_with_min} the maximum must
be computed:
\[
\left\Vert \boldsymbol{\epsilon}^{*}\right\Vert \geq\kappa\min_{c'\neq y}\max_{i}\left(\ln\left(-\frac{B_{c',i}+\sqrt{B_{c',i}^{2}-4A_{c',i}C_{c',i}}}{2A_{c',i}}\right)\right)>0.
\]

\subsection{Proof of \thmref{lower-bound-standard-RBF}\label{appendix:Proof-of-theorem-RBF}}

We now prove \thmref{lower-bound-standard-RBF} by extending the proof
of \Thmref{Robustness-with-individual-sigma}. For completeness we
restate the theorem:
\begin{thm}
If we use the standard RBF kernel (squared norm), then \Eqref{Lower_bound_perturbation_with_min}
becomes $\left\Vert \boldsymbol{\epsilon}^{*}\right\Vert \geq-\frac{\beta}{3}+\sqrt{\frac{\beta^{2}}{9}+\delta}>0$
with $\kappa=\frac{\sigma_{min}}{3}$ and $\beta=\max_{k}d\left(\mathbf{x},\mathbf{w}_{k}\right)$.
\end{thm}

The theorem states that the derived bound also holds for the frequently
used squared exponential kernel (Gaussian RBF), which implicitly projects
data into an infinite-dimensional space. In principle, any squared
distance induced by a norm can be used.
\begin{proof}
The proof follows the technique used to prove \Thmref{Robustness-with-individual-sigma}
with the following changes: We modify the initial lower bounds of
the distances. In \lemref{Lemma}, we used 
\[
d\left(\mathbf{x}+\boldsymbol{\epsilon},\mathbf{w}_{k}\right)\leq\left\Vert \boldsymbol{\epsilon}\right\Vert +d\left(\mathbf{x},\mathbf{w}_{k}\right)
\]
to derive the result. If we square this inequality, we get
\[
d^{2}\left(\mathbf{x}+\boldsymbol{\epsilon},\mathbf{w}_{k}\right)\leq\left\Vert \boldsymbol{\epsilon}\right\Vert ^{2}+d^{2}\left(\mathbf{x},\mathbf{w}_{k}\right)+2\left\Vert \boldsymbol{\epsilon}\right\Vert d\left(\mathbf{x},\mathbf{w}_{k}\right).
\]
Among all $d\left(\mathbf{x},\mathbf{w}_{k}\right)$ there exists
a maximum $\beta=\max_{k}d\left(\mathbf{x},\mathbf{w}_{k}\right)$.
Using this result, we get 
\[
d^{2}\left(\mathbf{x}+\boldsymbol{\epsilon},\mathbf{w}_{k}\right)\leq\left(\left\Vert \boldsymbol{\epsilon}\right\Vert ^{2}+2\left\Vert \boldsymbol{\epsilon}\right\Vert \beta\right)+d^{2}\left(\mathbf{x},\mathbf{w}_{k}\right)
\]
and we further we relax the bound to
\[
d^{2}\left(\mathbf{x}+\boldsymbol{\epsilon},\mathbf{w}_{k}\right)\leq\left(3\left\Vert \boldsymbol{\epsilon}\right\Vert ^{2}+2\left\Vert \boldsymbol{\epsilon}\right\Vert \beta\right)+d^{2}\left(\mathbf{x},\mathbf{w}_{k}\right).
\]
With this, we get for the lower bound of the detection probability
\eqref{lower_bound_detection_prob}
\[
\exp\left(-\frac{d^{2}\left(\mathbf{x}+\boldsymbol{\epsilon},\mathbf{w}_{k}\right)}{\sigma}\right)\geq\exp\left(-\frac{d^{2}\left(\mathbf{x},\mathbf{w}_{k}\right)}{\sigma}\right)\underbrace{\exp\left(-\frac{3\left\Vert \boldsymbol{\epsilon}\right\Vert ^{2}+2\beta\left\Vert \boldsymbol{\epsilon}\right\Vert }{\sigma}\right)}_{=:\exp(-z)}.
\]
Similarly, we can derive the following result for the squared triangle
inequality of $d\left(\mathbf{x},\mathbf{w}_{k}\right)\leq\left\Vert \boldsymbol{\epsilon}\right\Vert +d\left(\mathbf{x}+\boldsymbol{\epsilon},\mathbf{w}_{k}\right)$:
\begin{align*}
d^{2}\left(\mathbf{x},\mathbf{w}_{k}\right) & \leq\left\Vert \boldsymbol{\epsilon}\right\Vert ^{2}+d^{2}\left(\mathbf{x}+\boldsymbol{\epsilon},\mathbf{w}_{k}\right)+2\left\Vert \boldsymbol{\epsilon}\right\Vert d\left(\mathbf{x}+\boldsymbol{\epsilon},\mathbf{w}_{k}\right)\\
 & \leq\left\Vert \boldsymbol{\epsilon}\right\Vert ^{2}+d^{2}\left(\mathbf{x}+\boldsymbol{\epsilon},\mathbf{w}_{k}\right)+2\left\Vert \boldsymbol{\epsilon}\right\Vert \left(\left\Vert \boldsymbol{\epsilon}\right\Vert +d\left(\mathbf{x},\mathbf{w}_{k}\right)\right)\\
 & =\left\Vert \boldsymbol{\epsilon}\right\Vert ^{2}+d^{2}\left(\mathbf{x}+\boldsymbol{\epsilon},\mathbf{w}_{k}\right)+2\left\Vert \boldsymbol{\epsilon}\right\Vert ^{2}+2\left\Vert \boldsymbol{\epsilon}\right\Vert d\left(\mathbf{x},\mathbf{w}_{k}\right)\\
 & =\left\Vert \boldsymbol{\epsilon}\right\Vert ^{2}+d^{2}\left(\mathbf{x}+\boldsymbol{\epsilon},\mathbf{w}_{k}\right)+2\left\Vert \boldsymbol{\epsilon}\right\Vert ^{2}+2\beta\left\Vert \boldsymbol{\epsilon}\right\Vert \\
 & =\left(3\left\Vert \boldsymbol{\epsilon}\right\Vert ^{2}+2\beta\left\Vert \boldsymbol{\epsilon}\right\Vert \right)+d^{2}\left(\mathbf{x}+\boldsymbol{\epsilon},\mathbf{w}_{k}\right).
\end{align*}
This implies that
\begin{align*}
d^{2}\left(\mathbf{x}+\boldsymbol{\epsilon},\mathbf{w}_{k}\right) & \geq-\left(3\left\Vert \boldsymbol{\epsilon}\right\Vert ^{2}+2\beta\left\Vert \boldsymbol{\epsilon}\right\Vert \right)+d^{2}\left(\mathbf{x},\mathbf{w}_{k}\right)
\end{align*}
and for the upper bound of the detection probability \eqref{upper_bound_detection_prob}
\[
\exp\left(-\frac{d^{2}\left(\mathbf{x}+\boldsymbol{\epsilon},\mathbf{w}_{k}\right)}{\sigma}\right)\leq\exp\left(-\frac{d^{2}\left(\mathbf{x},\mathbf{w}_{k}\right)}{\sigma}\right)\underbrace{\exp\left(\frac{3\left\Vert \boldsymbol{\epsilon}\right\Vert ^{2}+2\beta\left\Vert \boldsymbol{\epsilon}\right\Vert }{\sigma}\right)}_{=\exp(z)}.
\]
By using the results of \lemref{Lemma}, we get for the root \eqref{epsilon_1_solution}
\[
3\left\Vert \boldsymbol{\epsilon}'_{0}\right\Vert ^{2}+2\beta\left\Vert \boldsymbol{\epsilon}'_{0}\right\Vert =\sigma\ln\left(-\frac{B+\sqrt{B^{2}-4AC}}{2A}\right)=\left\Vert \boldsymbol{\epsilon}_{0}\right\Vert >0.
\]
If we solve this equation for $\left\Vert \boldsymbol{\epsilon}'_{0}\right\Vert $,
we get 
\begin{equation}
\left\Vert \boldsymbol{\epsilon}'_{0}\right\Vert =-\frac{\beta}{3}+\sqrt{\frac{\beta^{2}}{9}+\frac{1}{3}\left\Vert \boldsymbol{\epsilon}_{0}\right\Vert }>0\label{eq:squared_case}
\end{equation}
for the valid solution as the negative part would lead to a negative
$\left\Vert \boldsymbol{\epsilon}'_{0}\right\Vert $. By following
the additional proof steps of \thmref{Robustness-with-individual-sigma},
we get the result. 
\end{proof}
It should be noted, that this result only considers correctly classified
samples. For incorrectly classified samples, the expression to determine$\left\Vert \boldsymbol{\epsilon}_{0}\right\Vert $(which
means $\delta$) becomes negative; hence, \eqref{squared_case} cannot
be computed since the quadratic equation could have no roots. Consequently,
\eqref{squared_case} cannot be used to formulate a closed-form loss
function for correctly and incorrectly classified samples. However,
for incorrectly classified samples, it is sufficient to optimize $\delta$
(i.\,e., \eqref{Lower_bound_perturbation}) and optimize \eqref{squared_case}
for correctly classified samples. The joint formulation takes the
form:
\begin{equation}
\begin{cases}
-\frac{\beta}{3}+\sqrt{\frac{\beta^{2}}{9}+\delta} & \textrm{if }p_{y}\left(\mathbf{x}\right)-p_{c'}\left(\mathbf{x}\right)>0,\\
\lambda\cdot\delta & \textrm{otherwise,}
\end{cases}\label{eq:loss-squared-distances}
\end{equation}
where $\lambda>0$ is a regularization factor to balance the two differently
scaled loss terms. 

For Gaussian kernel RBF networks, \eqref{squared_case} can be simplified
by \eqref{robust_RBF_loss_pos_reasoning_only}. This presents the
first result about robustness optimization of Gaussian kernel RBF
networks. In general, the term $\beta$ has a minor contribution to
the network optimization. Thus, similar to the non-squared case, it
is sufficient to only optimize $\delta$. However, to obtain a precise
robustness value via a margin loss formulation, \eqref{squared_case}
must be optimized.

\subsection{Proof of \thmref{lower-bound-tangent-distance}\label{appendix:Proof-of-theorem-TD}}

We now prove \thmref{lower-bound-tangent-distance} by extending the
proof of \Thmref{Robustness-with-individual-sigma}. For completeness
we restate the theorem:
\begin{thm*}
If we use the tangent distance in the RBF of \Eqref{rbf-kernel-definition},
then \Eqref{Lower_bound_perturbation_with_min} holds with $\kappa=\frac{1}{2}\sigma_{min}$
and $\left\Vert \cdot\right\Vert $ being the Euclidean norm.
\end{thm*}
See \appendixref{appendix_tangent_distance} for information about
the tangent distance.
\begin{proof}
The proof follows the technique used to prove \Thmref{Robustness-with-individual-sigma}
with the following changes: We modify the initial lower bounds of
the distances. The goal is to derive a lower bound for 
\[
d_{T}\left(\mathbf{x}+\boldsymbol{\epsilon},S\right)=d_{E}\left(\mathbf{x}+\boldsymbol{\epsilon},\mathbf{w}^{*}\left(\mathbf{x}+\boldsymbol{\epsilon}\right)\right),
\]
where $\mathbf{w}^{*}\left(\mathbf{x}\right)$ is the best approximating
element with respect to $\mathbf{x}$, see \eqref{TD_minimizer}.
Similar to before (see \eqref{lower_bound_detection_prob}) we apply
the triangle inequality to derive
\[
d_{E}\left(\mathbf{x}+\boldsymbol{\epsilon},\mathbf{w}^{*}\left(\mathbf{x}+\boldsymbol{\epsilon}\right)\right)\leq\left\Vert \boldsymbol{\epsilon}\right\Vert _{E}+d_{E}\left(\mathbf{x},\mathbf{w}^{*}\left(\mathbf{x}+\boldsymbol{\epsilon}\right)\right).
\]
Now, we upper bound $d_{E}\left(\mathbf{x},\mathbf{w}^{*}\left(\mathbf{x}+\boldsymbol{\epsilon}\right)\right)$
by applying the triangle inequality again:
\begin{align*}
d_{E}\left(\mathbf{x},\mathbf{w}^{*}\left(\mathbf{x}+\boldsymbol{\epsilon}\right)\right) & \leq d_{E}\left(\mathbf{x},\mathbf{w}^{*}\left(\mathbf{x}\right)\right)+d_{E}\left(\mathbf{w}^{*}\left(\mathbf{x}\right),\mathbf{w}^{*}\left(\mathbf{x}+\boldsymbol{\epsilon}\right)\right)\\
 & =d_{T}\left(\mathbf{x},S\right)+d_{E}\left(\mathbf{w}^{*}\left(\mathbf{x}\right),\mathbf{w}^{*}\left(\mathbf{x}+\boldsymbol{\epsilon}\right)\right).
\end{align*}
The expression $d_{E}\left(\mathbf{w}^{*}\left(\mathbf{x}\right),\mathbf{w}^{*}\left(\mathbf{x}+\boldsymbol{\epsilon}\right)\right)$
can be upper bounded by $\left\Vert \boldsymbol{\epsilon}\right\Vert _{E}$:
\begin{align*}
d_{E}\left(\mathbf{w}^{*}\left(\mathbf{x}\right),\mathbf{w}^{*}\left(\mathbf{x}+\boldsymbol{\epsilon}\right)\right) & =d_{E}\left(\mathbf{w}+\mathbf{W}\mathbf{W}^{\mathrm{T}}\left(\mathbf{x}-\mathbf{w}\right),\mathbf{w}+\mathbf{W}\mathbf{W}^{\mathrm{T}}\left(\mathbf{x}+\boldsymbol{\epsilon}-\mathbf{w}\right)\right)\\
 & =\left\Vert \mathbf{w}+\mathbf{W}\mathbf{W}^{\mathrm{T}}\left(\mathbf{x}-\mathbf{w}\right)-\mathbf{w}-\mathbf{W}\mathbf{W}^{\mathrm{T}}\left(\mathbf{x}+\boldsymbol{\epsilon}-\mathbf{w}\right)\right\Vert _{E}\\
 & =\left\Vert \mathbf{W}\mathbf{W}^{\mathrm{T}}\boldsymbol{\epsilon}\right\Vert _{E}\\
 & =\sqrt{\boldsymbol{\epsilon}^{\mathrm{T}}\mathbf{W}\underbrace{\mathbf{W}^{\mathrm{T}}\mathbf{W}}_{\mathbf{I}}\mathbf{W}^{\mathrm{T}}\boldsymbol{\epsilon}}\\
 & =\sqrt{\boldsymbol{\epsilon}^{\mathrm{T}}\mathbf{W}\mathbf{W}^{\mathrm{T}}\boldsymbol{\epsilon}}\\
 & =\left\Vert \mathbf{W}^{\mathrm{T}}\boldsymbol{\epsilon}\right\Vert _{E}.
\end{align*}
Next, we use the fact that the Euclidean norm is compatible with the
spectral norm and that the spectral norm of $\mathbf{W}^{\mathrm{T}}$
is 1 (because of the orthonormal basis assumption):
\[
d_{E}\left(\mathbf{w}^{*}\left(\mathbf{x}\right),\mathbf{w}^{*}\left(\mathbf{x}+\boldsymbol{\epsilon}\right)\right)=\left\Vert \mathbf{W}^{\mathrm{T}}\boldsymbol{\epsilon}\right\Vert _{E}\leq\underbrace{\left\Vert \mathbf{W}^{\mathrm{T}}\right\Vert _{E}}_{=1}\left\Vert \boldsymbol{\epsilon}\right\Vert _{E}=\left\Vert \boldsymbol{\epsilon}\right\Vert _{E}.
\]
Finally, combining the results we get
\[
d_{T}\left(\mathbf{x}+\boldsymbol{\epsilon},S\right)\leq2\left\Vert \boldsymbol{\epsilon}\right\Vert _{E}+d_{T}\left(\mathbf{x},S\right).
\]
Similarly, we can derive 
\[
d_{T}\left(\mathbf{x}+\boldsymbol{\epsilon},S\right)\geq-2\left\Vert \boldsymbol{\epsilon}\right\Vert _{E}+d_{T}\left(\mathbf{x},S\right).
\]
Using these two results in \eqref{lower_bound_detection_prob} and
\eqref{upper_bound_detection_prob}, we conclude that the robustness
lower bound for the tangent distance is given by
\[
\left\Vert \boldsymbol{\epsilon}^{*}\right\Vert _{E}\geq\frac{\sigma_{min}}{2}\min_{c'\neq y}\left(\ln\left(-\frac{B_{c'}+\sqrt{B_{c'}^{2}-4A_{c'}C_{c'}}}{2A_{c'}}\right)\right)>0.
\]
\end{proof}
Note that compared to \thmref{Robustness-with-individual-sigma},
we leverage the triangle inequality multiple times over the same expression.
Hence, it must be expected that the derived lower bound is not as
tight as for \thmref{Robustness-with-individual-sigma}. The same
applies to \thmref{lower-bound-standard-RBF}.

If we combine this result with \thmref{lower-bound-standard-RBF},
we can conclude that we have to divide the result of \thmref{lower-bound-standard-RBF}
by 2. This result can be obtained by substituting $2\left\Vert \boldsymbol{\epsilon}\right\Vert _{E}$
with a new variable and by following the proof steps of \thmref{lower-bound-standard-RBF}.

\section{Further Theoretical Results\label{appendix:Further-theoretical-results}}

\paragraph*{A motivating example for negative reasoning.}

In general, negative reasoning as the retrieval of evidence from absent
features is not only supported by results from cognitive science \citep{Hsu2017}
but can also be motivated by a thought experiment: Assume a fine-grained
multi-class classification problem. Between two close classes, $A$
and $B$, only the presence of one particular base feature discriminates
between the two classes (present for $A$ and not present for $B$).
Further, both classes are supported by $N$ detectable base features.
If only positive reasoning is allowed, classes $A$ and $B$ will
be supported by, at most, $N+1$ and $N$ features, respectively.
If all features contribute equally to the class evidence, the absence
of a base feature has a higher impact on class $B$ than it has on
class $A$. Consequently, how features could contribute to the class
evidence is not balanced. If negative reasoning were used, the problem
would be fixed since both classes would be supported by $N+1$ features
(presence or absence contributes). 

\paragraph*{When is $p_{c}(\mathbf{x})=0$ or $p_{c}(\mathbf{x})=1$?}

Only if the reasoning and the detection probabilities become crisp
(binary) vectors. This can be easily shown by considering the probability
tree diagram. Understanding when the ``optimal'' probability outputs
can be generated is important. It also implies that the classification
output cannot be more discriminating than the detection probability
function. Hence, if one wants to improve the classification power
of the model, the discrimination power of the detection probability
function must be improved.

\paragraph*{How do we initialize the $\sigma_{k}$ in \eqref{rbf-kernel-definition}
such that the gradients do not vanish during training?}

If $\sigma_{k}$ is not chosen correctly,, the network can be difficult
to train because of vanishing gradients. To avoid this, we propose
the following initialization strategy: The idea is to compute how
much the distances between data points vary and to select $\sigma_{k}$
such that this value is mapped to a small probability. Assume that
the mean distance is denoted by $mean$ and the standard deviation
of the distances is denoted by $std$, then $\sigma_{k}$ can be initialized
by 
\begin{equation}
\sigma=-\frac{mean+std}{\ln\left(p_{0}\right)},\label{eq:sigma_initial}
\end{equation}
whereby $p_{0}$ is a defined lower bound for the expected detection
similarity (e.\,g., $p_{0}=0.01$). Moreover, the concepts \citet{GhiasiShirazi2019}
developed can also be applied.

\paragraph*{Relation to Generalized Learning Vector Quantization \citep{Sato1996}.}

It turns out that our CBC is equivalent to GLVQ under certain circumstances.
More precisely, if the reasoning becomes crisp and is only driven
by positive reasoning, it constitutes a one-hot vector coding of class
responsibility such that the CBC components realize class-specific
GLVQ prototypes. In this situation, the CBC optimization of the probability
gap yields the optimization of a scaled hypothesis margin in GLVQ:
\[
\mathrm{HypothesisMargin\left(\mathbf{x}\right)=\frac{1}{2}}\left(d_{E}\left(\mathbf{x},\mathbf{w}_{k_{c'}}\right)-d_{E}\left(\mathbf{x},\mathbf{w}_{k_{y}}\right)\right),
\]
where $\mathbf{w}_{k_{c'}}$ is the best matching prototype of an
incorrect class, and $\mathbf{w}_{k_{y}}$ is the best matching prototype
of the correct class. Consequently, the optimization of the probability
gap optimizes for robustness. At the same time, our derived robust
loss formulation \eqref{robust_RBF_loss_pos_reasoning_only} simplifies
to a scaled hypothesis margin as well. Assume that $k_{y}$ is the
one-hot index of $\mathbf{v}_{y}$ and $k_{c'}$ is the index of $\mathbf{v}_{c'}$.
Then, \eqref{robust_RBF_loss_pos_reasoning_only} becomes\emph{ }(only
considering the logarithm)\emph{
\begin{align*}
\ln\left(\frac{\mathbf{v}_{y}^{\mathrm{T}}\mathbf{d}\left(\mathbf{x}\right)}{\mathbf{v}_{c'}^{\mathrm{T}}\mathbf{d}\left(\mathbf{x}\right)}\right) & =\ln\left(\mathbf{v}_{y}^{\mathrm{T}}\mathbf{d}\left(\mathbf{x}\right)\right)-\ln\left(\mathbf{v}_{c'}^{\mathrm{T}}\mathbf{d}\left(\mathbf{x}\right)\right)\\
 & =\ln\left(d_{k_{y}}\left(\mathbf{x}\right)\right)-\ln\left(d_{k_{c'}}\left(\mathbf{x}\right)\right)\\
 & =\frac{d_{E}\left(\mathbf{x},\mathbf{w}_{k_{c'}}\right)}{\sigma_{k_{c'}}}-\frac{d_{E}\left(\mathbf{x},\mathbf{w}_{k_{y}}\right)}{\sigma_{k_{y}}}.
\end{align*}
}

\paragraph*{If the requiredness is uncertain the output probability will it be
as well.}

If $\mathbf{r}_{c}=\frac{1}{2}\mathbf{1}$, then $p_{c}(x)=\frac{1}{2}$.
This is an interesting result as it states that if there is no tendency
in the requiredness, there is also uncertainty in the output probability,
no matter how good the detection is or what the prior learns. This
also states that the network could produce a constant output. In practice,
we have never observed this behavior.

\section{Extended Experimental Results\label{appendix:Extended-experimental-results}}

This section presents the extended experimental results. For training
and evaluation, we used the following hardware and software:
\begin{itemize}
\item Nvidia Tesla V100 GPU with 32\,GB memory;
\item Intel Xeon Silver 4114 (2.20\,GHz) CPU with 128\,GB memory;
\item Ubuntu Focal Fossa (20.04 LTS) operating system;
\item Python 3.10.11;
\item PyTorch 2.4.0 with CUDA 12.1.
\end{itemize}

\subsection{Interpretability and performance assessment: Comparison with PIPNet\label{appendix:Interpretability-assessment:-PIPNet}}

\begin{table}
\begin{centering}
\begin{tabular}{ccccccc}
\toprule 
 & \multicolumn{2}{c}{CUB} & \multicolumn{2}{c}{CARS} & \multicolumn{2}{c}{PETS}\tabularnewline
 & ResNet50 & ConvNeXt & ResNet50 & ConvNeXt & ResNet50 & ConvNeXt\tabularnewline
\midrule
\midrule 
number of components & 2000 & 768 & 2000 & 768 & 2000 & 768\tabularnewline
batch size pre-training & 80 & 128 & 80 & 128 & 80 & 128\tabularnewline
batch size fine-tuning and end-to-end & 64 & 64 & 64 & 64 & 64 & 64\tabularnewline
learning rate pre-training & $2.25e^{-4}$ & $3.75e^{-4}$ & $2.25e^{-4}$ & $2.25e^{-4}$ & $6.25e^{-5}$ & $1.75e^{-4}$\tabularnewline
learning rate fine-tuning & $4.50e^{-4}$ & $5.00e^{-3}$ & $2.25e^{-4}$ & $5.00e^{-3}$ & $1.25e^{-4}$ & $5.00e^{-3}$\tabularnewline
learning rate end-to-end & $2.25e^{-4}$ & $3.75e^{-4}$ & $2.25e^{-4}$ & $2.25e^{-4}$ & $6.25e^{-5}$ & $1.75e^{-4}$\tabularnewline
epochs pre-training & 14 & 12 & 14 & 12 & 14 & 12\tabularnewline
epochs fine-tuning & 96 & 48 & 96 & 48 & 96 & 48\tabularnewline
epochs end-to-end & 144 & 84 & 144 & 84 & 144 & 84\tabularnewline
\bottomrule
\end{tabular}
\par\end{centering}
\caption{Deep CBC training parameters.\label{tab:Deep-CBC-training}}

\end{table}
We trained our deep CBCs by following the protocol of \citet{Nauta2023}.
This includes
\begin{itemize}
\item pre-training with different loss functions (self-supervised and supervised),
\item training with two different supervised training stages consisting
of only fine tuning the reasoning probabilities followed by partial
training of our CBC head with backbone layers (single Adam optimizer),
and
\item only optimizing the margin loss with $\gamma=0.025$.
\end{itemize}
The remaining parameters can be found in \tabref{Deep-CBC-training}.
In the following, we present the complete accuracy comparison with
PIPNet and ProtoPool. Then, we analyze the interpretability of PIPNet
on other samples to demonstrate that the interpretation can be misleading.
After that, we analyze the learned components of a CBC and discuss
the limitations of the interpretability of CBC\@. It should be noted
that the deep CBC models are only partially interpretable since the
feature backbone is a black box. Therefore, it is not always possible
to explain the reasoning process conclusively, as shown in the last
experiment.

\paragraph{ConvNeXt is better than ResNet in terms of accuracy.}

\begin{table}
\begin{centering}
\begin{tabular}{cccc}
\toprule 
 & CUB & CARS & PETS\tabularnewline
\midrule
\midrule 
PIPNet-C & $84.3\pm0.2$ & $88.2\pm0.5$ & $92.0\pm0.3$\tabularnewline
PIPNet-R & $82.0\pm0.3$ & $86.5\pm0.3$ & $88.5\pm0.2$\tabularnewline
CBC pos.~reas. & $28.6\pm0.8$ & $25.3\pm2.3$ & $69.5\pm5.1$\tabularnewline
ProtoPool & $85.5\pm0.1$ & $88.9\pm0.1$ & $87.2^{*}\pm0.1$\tabularnewline
ProtoViT (CaiT-XXS 24) & $85.8\pm0.2$ & $92.4\pm0.1$ & $93.3^{*}\pm0.2$\tabularnewline
CBC-C & $\mathbf{87.8\pm0.1}$ & $\mathbf{93.0\pm0.0}$ & $\mathbf{93.9\pm0.1}$\tabularnewline
CBC-R & $83.3\pm0.3$ & $92.7\pm0.1$ & $90.1\pm0.1$\tabularnewline
CBC-R Full & $82.8\pm0.3$ & $92.8\pm0.1$ & $89.5\pm0.2$\tabularnewline
\bottomrule
\end{tabular}
\par\end{centering}
\caption{Accuracy results with different models.\label{tab:Accuracy-results-CBC-ProtoPool-PIPNet}}
\end{table}
Like the PIPNet experiments, we trained our CBC with a ConvNeXt-tiny
(denoted by C) and ResNet50 (denoted by R) backbone. Additionally,
we analyzed the impact of full backbone training instead of only training
a few last layers (denoted by CBC-R Full). \tabref{Accuracy-results-CBC-ProtoPool-PIPNet}
presents the accuracy results of our method in comparison with ProtoPool
and PIPNet. 

First, it must be noted that our approach outperforms all the other
methods and, hence, sets a new benchmark performance. The full training
of a network is less effective than partial training (cf.\ CBC-R
Full with CBC-R). Consequently, training only the few layers selected
by \citet{Nauta2023} is sufficient.

\begin{figure}
\begin{centering}
\includegraphics[width=0.7\textwidth]{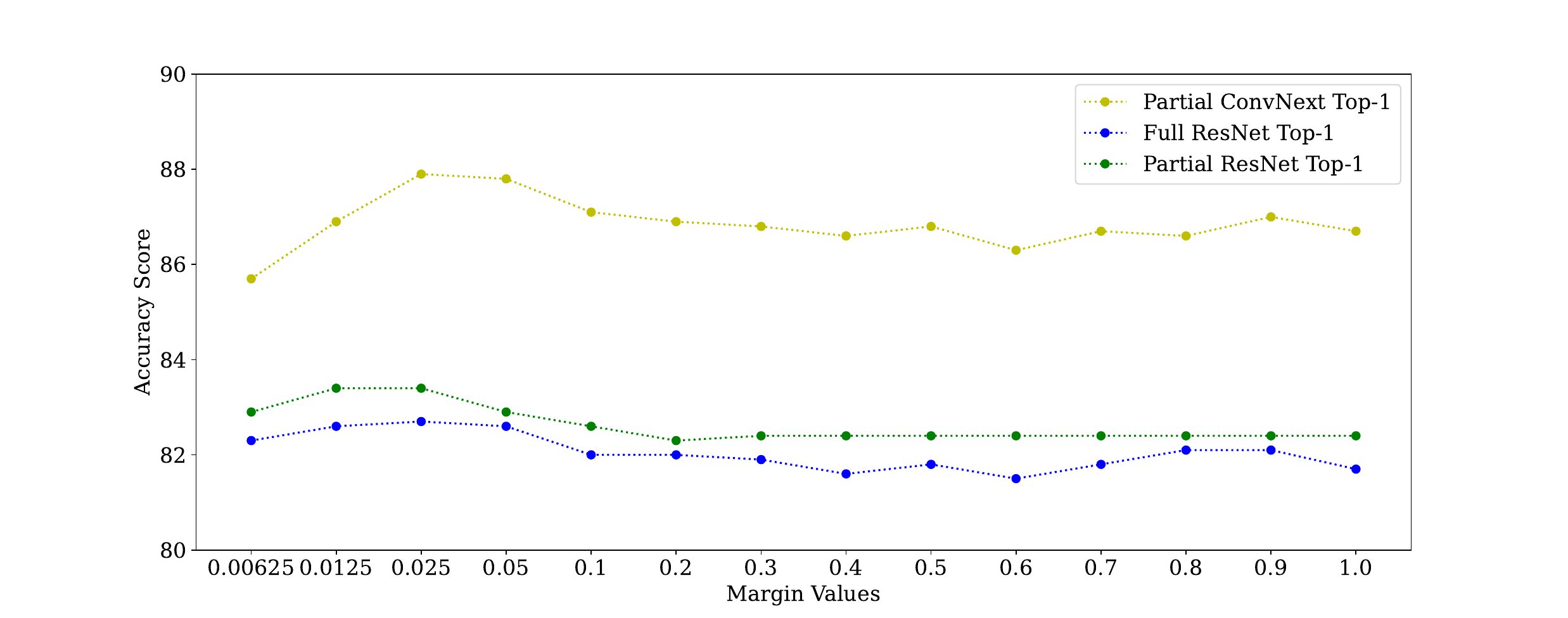}
\par\end{centering}
\caption{Margin value hyperparameter search for different backbone architectures.\label{fig:Margin-value-analysis.}}

\end{figure}
We also analyzed the impact of the margin parameter on the model training.
For this, we trained several CBCs with different margin values. \figref{Margin-value-analysis.}
summarizes the results. As we can see, the margin value has an impact
on the achievable performance, but the parameter is not critical with
respect to training stability. Based on this result, we selected our
chosen margin value of $0.025$. We performed a similar analysis with
the alignment loss value. It neither improved the accuracy nor changed
the top-10 component visualizations from the training dataset.

\paragraph{Analysis of interpretability with PIPNet.}

\begin{figure}
\begin{centering}
\includegraphics[width=0.8\textwidth]{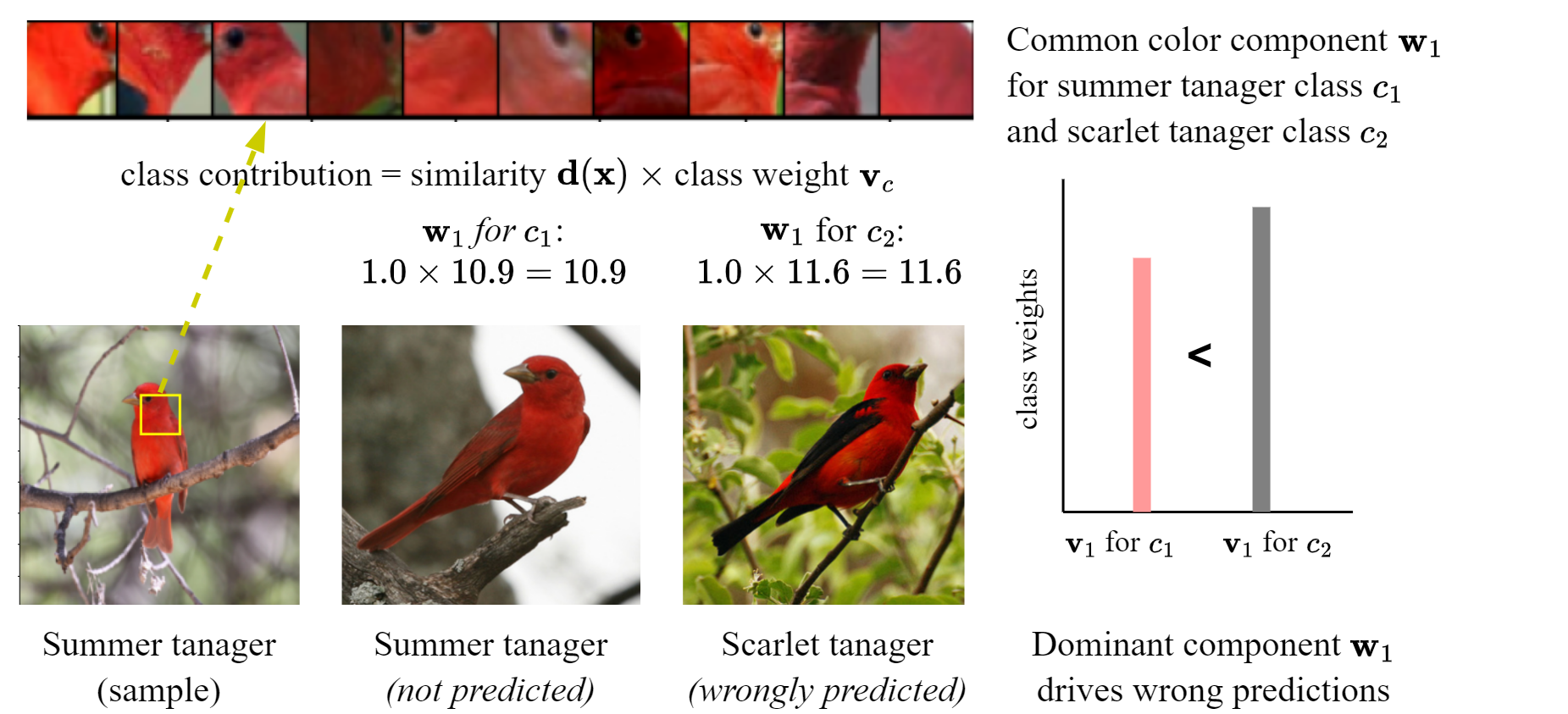}
\par\end{centering}
\caption{Issues of PIPNet with component sparsity while differentiating classes.\label{fig:Summer-tanager-vs.scarlet}}
\end{figure}
\begin{figure}
\begin{centering}
\includegraphics[width=0.7\textwidth]{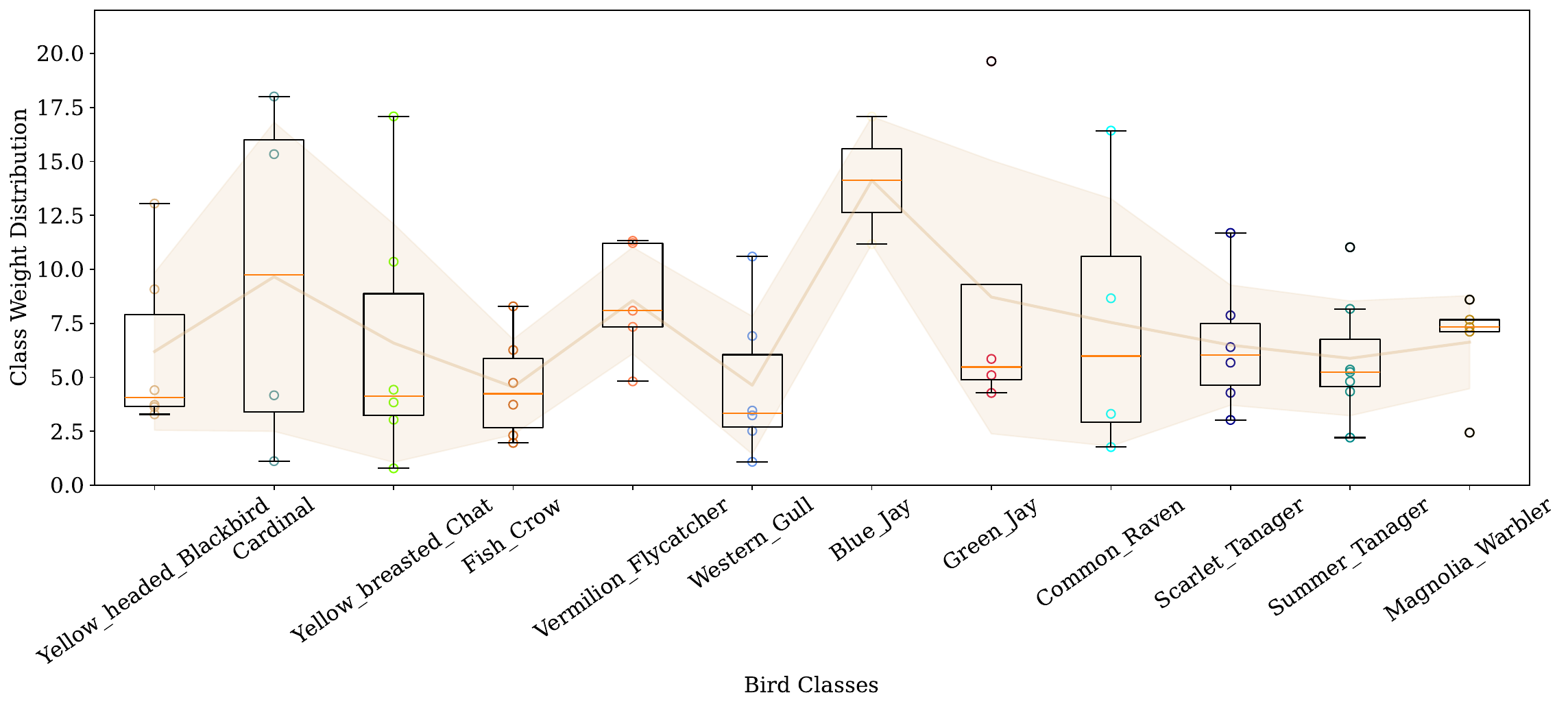}
\par\end{centering}
\caption{Unbalanced weight distribution of PIPNet. The image shows the box
plots of the weights of different classes, including the mean values
(solid line) and the standard deviation (shaded area). \label{fig:Unbalanced-weight-distribution}}
\end{figure}
Here, we extend our discussion on the aspects of interpretability
with respect to PIPNet. First, with \figref{Summer-tanager-vs.scarlet},
we demonstrate that enforced sparsity without a constraint over the
weights can be problematic. For example, scarlet and summer tanagers
have similar head color pattern regions but differently colored winged
feathers. Because of the artificially enforced sparsity, a commonly
shared component becomes more relevant for scarlet tanager, leading
to image misclassification. Note that such a result is only possible
because the weights are unbalanced. If the weights were constrained
to probability vectors, there would likely be a tie. \figref{Unbalanced-weight-distribution}
further analyzes this issue. Here, we plot the statistics of the weights
of different classes of PIPNet. The plot shows that PIPNet uses highly
different weight statistics to classify classes, which further provides
evidence for the hypothesized issue in \secref{Review-of-Prototype-based}
that less important prototypes are overemphasized.

\paragraph{Learning contextually relevant components.}

\begin{figure}
\begin{centering}
\includegraphics[width=0.8\textwidth]{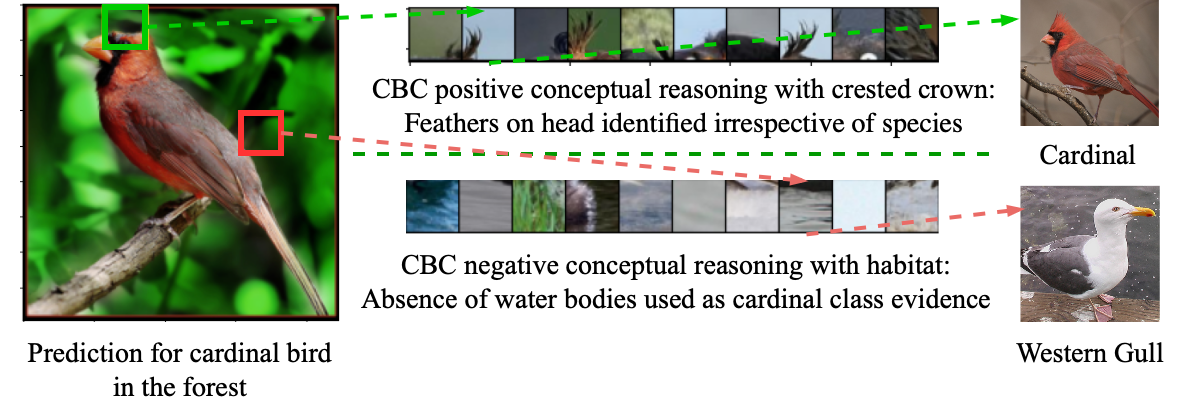}
\par\end{centering}
\caption{CBCs learning contextually relevant components for positive and negative
reasoning.\label{fig:cardinal_positive_and_negative_r}}
\end{figure}
We utilize \figref{cardinal_positive_and_negative_r} to highlight
an interesting property of CBCs: learning contextually relevant information
for positive and negative reasoning. Here, for example, when classifying
the given sample as a cardinal bird, the head region is selected and
compared with similarly extended features around the head region for
positive reasoning, and, for comparison, the head feather features
learned by the component are independent of the bird species. For
negative reasoning, the absence of information on water bodies from
the background is used to create evidence that the input cardinal
bird sample is found in non-coastal regions like forests. Thus, we
observe that CBC learns to exploit background information to make
predictions based on the input data distribution trends, such as cardinal
birds often having non-coastal regions as background in the CUB dataset.

\paragraph{Limitations of interpretability with deep CBCs. }

\begin{figure}
\begin{centering}
\includegraphics[width=1\textwidth]{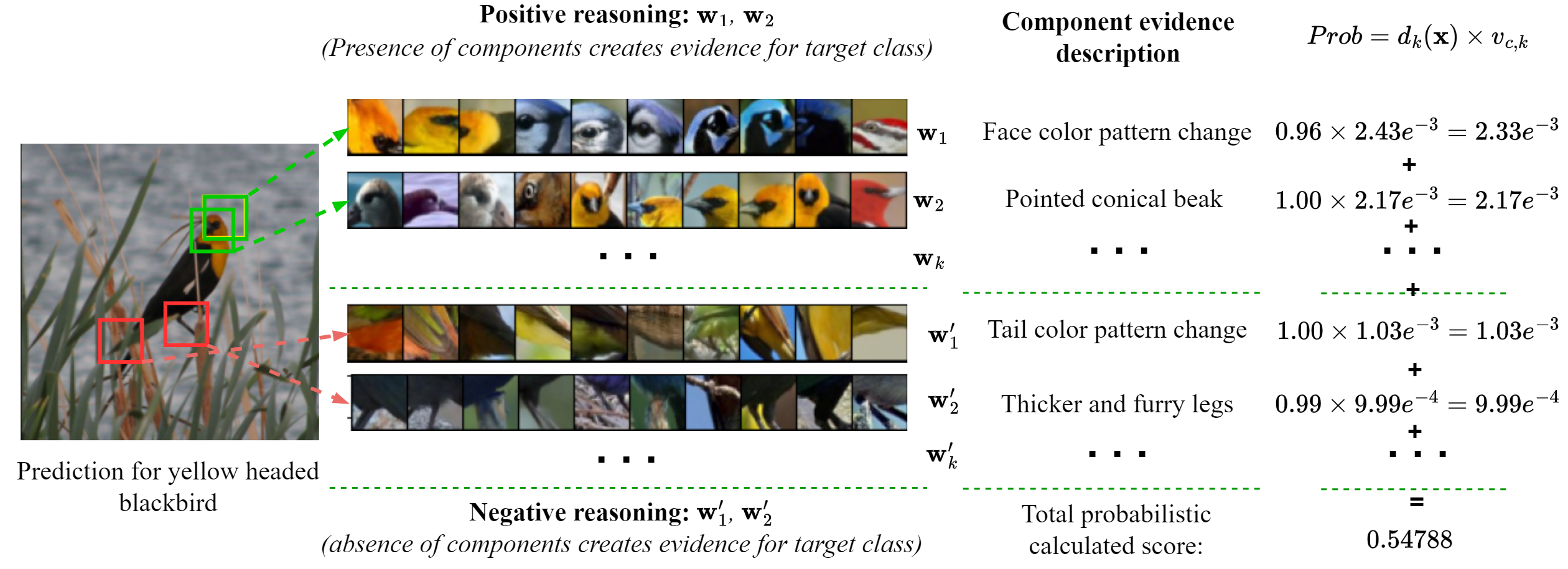}
\par\end{centering}
\caption{Probabilistic prediction mechanism of CBCs with positive and negative
component contributions.\label{fig:yellow_head_blackbird_prediction}}
\end{figure}
\begin{figure}
\begin{centering}
\includegraphics[width=0.8\textwidth]{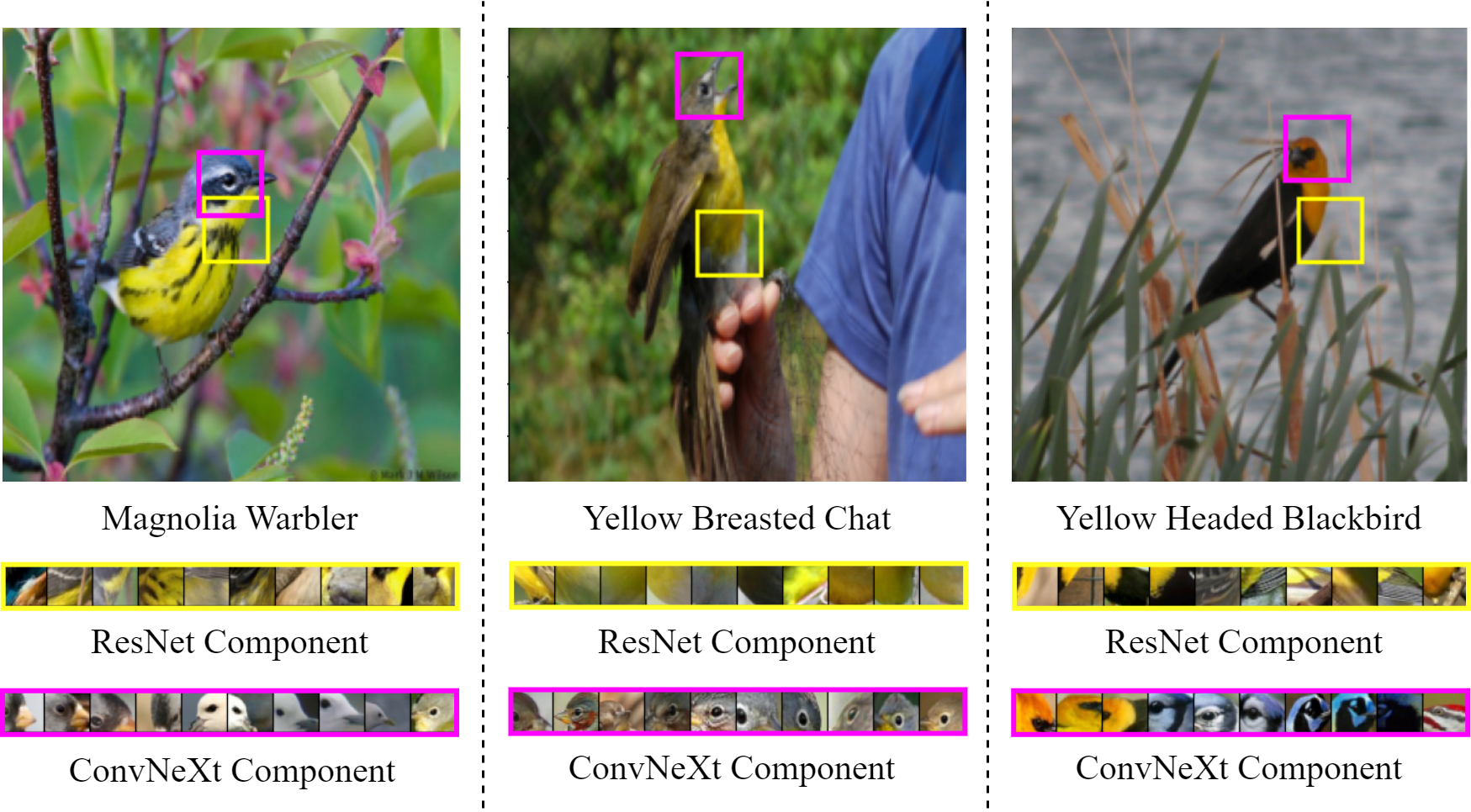}
\par\end{centering}
\caption{Comparing positive reasoning components for ResNet and ConvNext feature
extractor backbones for CBCs.\label{fig:resnet-and-convnext-comparison}}
\end{figure}
In \figref{yellow_head_blackbird_prediction}, we highlight the numerical
mechanism behind the qualitative analysis we presented for CBCs. We
observe that the learned CBC component represents contextually relevant
information independent of the bird species. However, we also state
that these component representations have \emph{non-interpretable}
aspects. For example, we cannot explain why the first positive component
is more important than the second one. And, in the case of bird samples
and background features overlapping rectangular patches, we are often
not sure whether the network utilized the bird or background features
as prediction contributors. We also observe that individual negative
reasoning component magnitudes are often less than individual positive
components, but still, their contribution is highly important in making
correct predictions. For example, we observe that the ResNet50 feature
extractor with 2000 components compared to 768 components of ConvNeXt-tiny
rarely uses negative reasoning. This frequently results in less reliable
predictions for ResNet50 compared to ConvNeXt-tiny, which leverages
negative reasoning. Moreover, with additional components, ResNet50
tends to learn color region-like components similar to PIPNet in addition
to learning contextually relevant components like CBC, as demonstrated
in \figref{resnet-and-convnext-comparison}. But, as evidenced by
our interpretability analysis and higher performance by ConvNeXt-tiny
backbone, these similar color region components are less reliable
than negative reasoning components for the prediction task.

\paragraph{Analysis of the prediction process for two similar classes with deep
CBC. }

\begin{figure}
\begin{centering}
\includegraphics[width=1\textwidth]{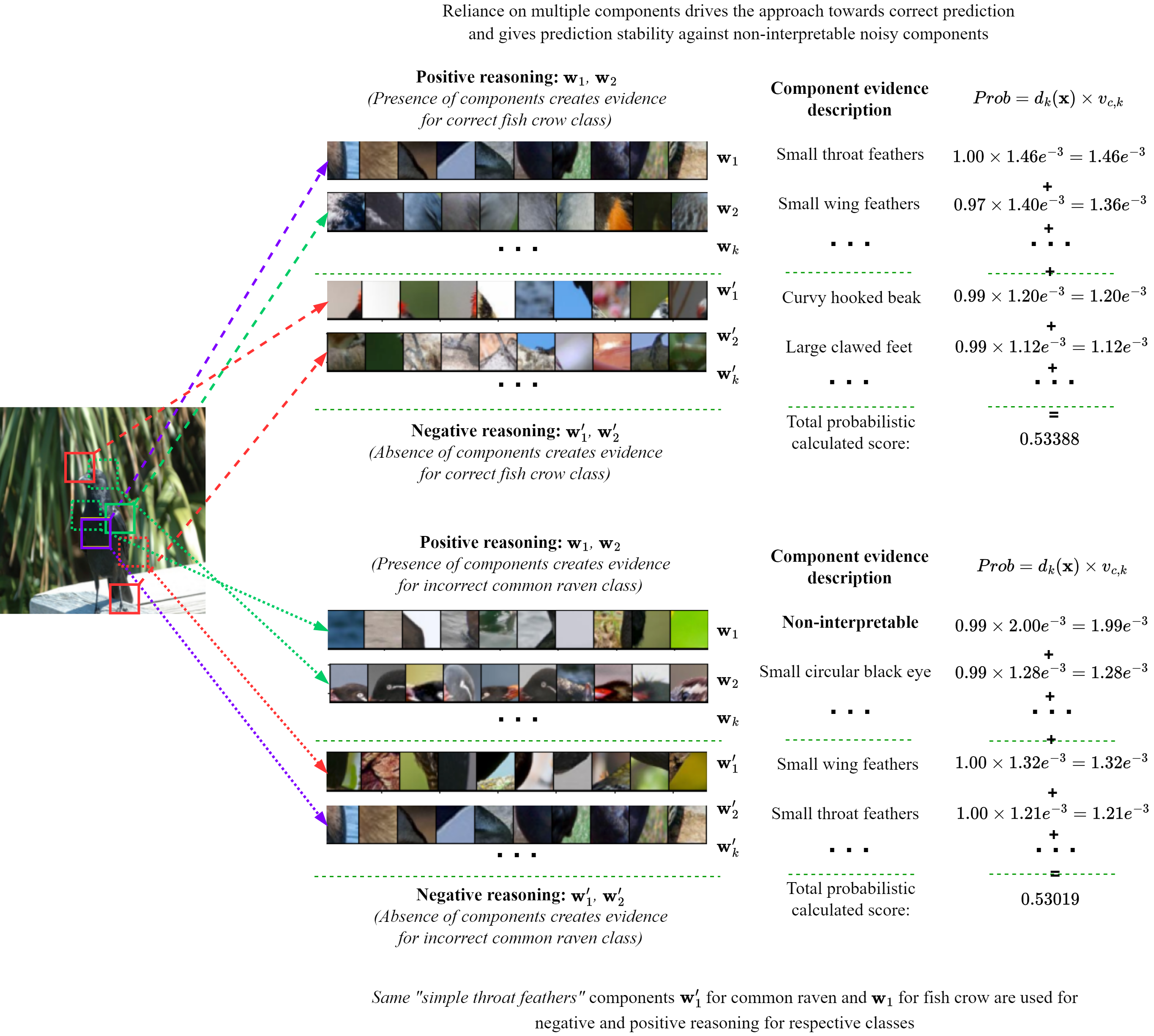}
\par\end{centering}
\caption{Probabilistic prediction mechanism analysis of the CBC approach for
the classes fish crow and common raven given a fish crow sample.\label{fig:crow-vs-raven-showdown}}

\end{figure}

Finally, in \figref{crow-vs-raven-showdown}, we summarize all the
relevant aspects of the prediction mechanism of the CBC approach.
Here, we observe that our approach analyzes all the critical features
needed to distinguish between crow and raven classes, like beak type,
feather feature patterns, and foot claw variations. For the raven
class, these features are more pronounced than those of crows. For
negative reasoning components, given that CBCs predict for the crow
class, we observe that hook-like curvy beaks and large clawed feet
are used as the top negative reasoning features to create evidence
for the crow class with the absence of such features in the input
sample. For positive reasoning components, we observe that smooth
feather patterns around the throat and wing region are used to create
evidence for the crow class. 

We also analyze the components of CBC prediction for wrong raven class
prediction given the crow sample. In this case, we observe that the
same smooth throat region component used to create positive reasoning
evidence for the crow class is now used for negative reasoning to
generate evidence for the raven class. Another component highlighting
smooth and smaller wing feathers is again used for negative reasoning,
which is opposite to what we observe for the crow class prediction.
Also, for positive reasoning, we observe that a non-interpretable
component is used for raven class prediction as the top contributor
with the highest magnitude among correct crow class prediction and
wrong raven class prediction components. However, the weight summation
constraint for probabilistic interpretations and reliance on multiple
components for predictions assist our methods to reduce misclassification.
Empirically, reliance on several components for both positive and
negative reasoning helps our approach to make robust predictions even
when noisy and non-interpretable components are included as the prediction
contributors.

\newpage{}

\subsection{Comparison with shallow models\label{appendix:Original-CBC-vs.our-CBC}}

We trained and evaluated all the models on the official MNIST training
and test dataset using the following setting:
\begin{itemize}
\item epochs: 40;
\item batch size: 128;
\item one $\sigma$ or several trainable $\sigma_{k}$ with initial values
of $58$ determined by \eqref{sigma_initial};
\item margin value $\gamma$ of 0.3 in \eqref{margin_loss} as proposed
by \citet{Saralajew2019};
\item margin value $\gamma$ of 1.58 for robustified training (margin loss
over the robust loss), which equals the commonly selected attack strength
for MNIST \citep[e.\,g.,][]{Voracek2022};
\item number of components or prototypes: 20;
\item number of prototypes or reasoning probabilities per class: 2;
\item subspace dimension $r=12$ for tangent distance \eqref{Tangent_Distance}
models;
\item AutoAttack \citep{Croce2020} with the standard setting and $\epsilon=\left\{ 0.5,1,1.58\right\} $;
\item \emph{no} data augmentation;
\item Adam optimizer \citep{Kingma2015} with learning rate of 0.005;
\item all components (prototypes), except affine subspace components (prototypes),
are constrained to be from$\left[0,1\right]^{28\cdot28}$ via clipping
after each update step (all learned components and prototypes are
valid images);
\item the basis representations $\mathbf{W}$ of the tangent distance are
parameterized by the approach of \citet{Mathiasen2020};
\item MNIST unit8 images are converted to float $\left[0,1\right]^{28\cdot28}$
by dividing by 255;
\item all parameters are \emph{initialized with random numbers} from a uniform
distribution on the interval $[0,1)$.
\end{itemize}
To analyze our proposed models and derived theorems, we trained each
model with the following distances:
\begin{itemize}
\item squared and non-squared Euclidean distance;
\item squared and non-squared tangent distance (abbreviated by TD in \tabref{Shallow-results-long-individual-sigma}).
\end{itemize}
Moreover, we use the following baseline models with the following
settings:
\begin{itemize}
\item \textbf{GLVQ} \citep{Sato1996}: This standard prototype-based model
gives the baseline for prototype-based learning (prototypes are preassigned
to classes and the winner-takes-all rule is applied). \textbf{GTLVQ}
\citep{Saralajew2020} is the GLVQ version with the tangent distance.
The models are trained by minimizing the GLVQ loss \citep{Sato1996}. 
\item \textbf{RBF} \citep{Broomhead1988}: Since our models are closely
related to RBF networks, we use these models to determine the benchmark
for our models without interpretability constraint. The models are
trained by the cross-entropy loss and the components are considered
as trainable parameters (update via stochastic gradient descent).
\item \textbf{Original CBC} \citep{Saralajew2019}: Our network is an extension
of this approach. Hence, we use this network type to analyze the impact
of our extension. The models are trained with the margin loss.
\end{itemize}
We evaluate the following proposed models:
\begin{itemize}
\item \textbf{CBC}: Our proposed CBC extension trained with margin loss.
\item \textbf{RBF-norm}: An RBF network where we constrain the class weights
to probability vectors. By this step, an RBF network becomes a CBC
where only positive reasoning is used which is interpretable. Similar
to standard RBF networks, we train this model with the cross-entropy
loss.
\item \textbf{Robust CBC}: Our proposed CBC extension with the robustness
loss is clipped at the respective margin.
\item \textbf{Robust RBF}: Similar to RBF-norm but trained with the proposed
robustness loss.
\end{itemize}
With this setting, \tabref{Shallow-results-long-individual-sigma}
presents the full version of \tabref{Shallow-results-short.}, where
we used component-wise, trainable temperatures. In the following,
we extend the discussion of the results by considering the evaluation
of the models against the full set of baselines. Moreover, to analyze
certain observations, we also computed the results for models where
the trainable temperature was shared among all components (see \tabref{Shallow-results-long-one-sigma})
and where we scaled the robustness loss \eqref{loss-squared-distances}
differently (see \tabref{Shallow-results-long-one-sigma-unscaled-loss}).
Additionally, we show in the last experiment how a patch-component-based
model can be created. We interpret the reasoning process of this model
and show how it automatically learns the two different concepts of
the digit seven. It should be noted that the interpretation of these
shallow models relies on a suitable visualization (representation
of the extracted information) for the end user. Moreover, if the model
becomes too large, then, similar to decision trees, the interpretation
can be complicated.

\paragraph{Negative reasoning improves accuracy. }

\begin{table}
\begin{centering}
\begin{tabular}{cccccccc}
\toprule 
 & \multirow{2}{*}{Accuracy} & \multicolumn{2}{c}{$\epsilon=0.5$} & \multicolumn{2}{c}{$\epsilon=1.0$} & \multicolumn{2}{c}{$\epsilon=1.58$}\tabularnewline
 &  & Emp.~Rob. & Cert.~Rob. & Emp.~Rob. & Cert.~Rob. & Emp.~Rob & Cert.~Rob.\tabularnewline
\midrule
\multirow{2}{*}{GLVQ} & $79.5\pm0.5$ & $70.2\pm0.5$ & $56.4\pm0.4$ & $58.8\pm0.3$ & $32.4\pm0.3$ & $44.1\pm0.3$ & $14.7\pm0.3$\tabularnewline
 & $80.5\pm0.6$ & $71.2\pm0.6$ & $57.0\pm0.4$ & $59.6\pm0.3$ & $32.3\pm0.3$ & $44.6\pm0.4$ & $14.0\pm0.4$\tabularnewline
\cmidrule{2-8} \cmidrule{3-8} \cmidrule{4-8} \cmidrule{5-8} \cmidrule{6-8} \cmidrule{7-8} \cmidrule{8-8} 
\multirow{2}{*}{RBF} & $85.2\pm0.7$ & $73.0\pm0.7$ & $-$ & $57.0\pm0.6$ & $-$ & $38.2\pm1.3$ & $-$\tabularnewline
 & $92.2\pm0.1$ & $82.0\pm0.1$ & $-$ & $61.9\pm0.9$ & $-$ & $29.9\pm1.7$ & $-$\tabularnewline
\cmidrule{2-8} \cmidrule{3-8} \cmidrule{4-8} \cmidrule{5-8} \cmidrule{6-8} \cmidrule{7-8} \cmidrule{8-8} 
\multirow{2}{*}{original CBC} & $70.0\pm1.6$ & $59.5\pm2.4$ & $-$ & $48.0\pm3.2$ & $-$ & $36.0\pm2.3$ & $-$\tabularnewline
 & $81.8\pm2.0$ & $72.9\pm1.5$ & $-$ & $62.0\pm1.0$ & $-$ & $46.9\pm0.4$ & $-$\tabularnewline
\cmidrule{2-8} \cmidrule{3-8} \cmidrule{4-8} \cmidrule{5-8} \cmidrule{6-8} \cmidrule{7-8} \cmidrule{8-8} 
\multirow{2}{*}{GTLVQ} & $92.9\pm0.4$ & $87.6\pm0.4$ & $81.5\pm0.7$ & $79.7\pm0.7$ & $63.2\pm1.1$ & $66.2\pm0.7$ & $\mathbf{37.3\pm0.9}$\tabularnewline
 & $94.1\pm0.3$ & $89.2\pm0.5$ & $\mathbf{83.0\pm0.7}$ & $81.4\pm0.6$ & $\mathbf{63.7\pm1.3}$ & $\mathbf{67.8\pm0.8}$ & $35.5\pm1.0$\tabularnewline
\cmidrule{2-8} \cmidrule{3-8} \cmidrule{4-8} \cmidrule{5-8} \cmidrule{6-8} \cmidrule{7-8} \cmidrule{8-8} 
\multirow{2}{*}{RBF TD} & $96.5\pm0.1$ & $91.7\pm0.1$ & $-$ & $81.8\pm0.2$ & $-$ & $60.0\pm0.5$ & $-$\tabularnewline
 & $\mathbf{97.8\pm0.1}$ & $\mathbf{92.9\pm0.2}$ & $-$ & $80.0\pm0.4$ & $-$ & $48.0\pm1.4$ & $-$\tabularnewline
\cmidrule{2-8} \cmidrule{3-8} \cmidrule{4-8} \cmidrule{5-8} \cmidrule{6-8} \cmidrule{7-8} \cmidrule{8-8} 
\multirow{2}{*}{original CBC TD} & $92.5\pm0.1$ & $87.1\pm0.2$ & -- & $78.8\pm0.3$ & -- & $65.2\pm0.4$ & --\tabularnewline
 & $95.0\pm0.4$ & $90.5\pm0.6$ & $-$ & $\mathbf{82.7\pm0.8}$ & $-$ & $67.0\pm0.6$ & $-$\tabularnewline
\midrule
\multirow{2}{*}{CBC} & $77.6\pm0.6$ & $66.5\pm0.5$ & $44.8\pm0.8$ & $54.6\pm0.4$ & $23.8\pm0.4$ & $41.9\pm0.3$ & $9.1\pm0.3$\tabularnewline
 & $87.4\pm0.3$ & $79.1\pm0.6$ & $32.3\pm1.8$ & $68.1\pm0.7$ & $0.2\pm0.1$ & $52.2\pm0.6$ & $0.0\pm0.0$\tabularnewline
\cmidrule{2-8} \cmidrule{3-8} \cmidrule{4-8} \cmidrule{5-8} \cmidrule{6-8} \cmidrule{7-8} \cmidrule{8-8} 
\multirow{2}{*}{RBF-norm} & $72.3\pm0.2$ & $63.9\pm0.1$ & $47.8\pm0.2$ & $54.3\pm0.2$ & $25.5\pm0.2$ & $41.6\pm0.1$ & $9.5\pm0.1$\tabularnewline
 & $77.3\pm0.2$ & $68.3\pm0.1$ & $27.8\pm0.2$ & $57.7\pm0.2$ & $0.7\pm0.0$ & $43.4\pm0.2$ & $0.0\pm0.0$\tabularnewline
\cmidrule{2-8} \cmidrule{3-8} \cmidrule{4-8} \cmidrule{5-8} \cmidrule{6-8} \cmidrule{7-8} \cmidrule{8-8} 
\multirow{2}{*}{CBC TD} & $92.5\pm0.1$ & $87.1\pm0.2$ & $54.0\pm0.6$ & $78.8\pm0.4$ & $2.3\pm0.5$ & $65.3\pm0.2$ & $0.0\pm0.0$\tabularnewline
 & $\mathbf{95.9\pm0.1}$ & $\mathbf{91.9\pm0.2}$ & $0.0\pm0.0$ & $\mathbf{84.5\pm0.2}$ & $0.0\pm0.0$ & $\mathbf{68.5\pm0.4}$ & $0.0\pm0.0$\tabularnewline
\cmidrule{2-8} \cmidrule{3-8} \cmidrule{4-8} \cmidrule{5-8} \cmidrule{6-8} \cmidrule{7-8} \cmidrule{8-8} 
\multirow{2}{*}{RBF-norm TD} & $90.2\pm0.1$ & $84.1\pm0.3$ & $46.1\pm0.4$ & $75.6\pm0.2$ & $6.8\pm1.0$ & $61.9\pm0.2$ & $0.0\pm0.0$\tabularnewline
 & $92.1\pm0.2$ & $86.4\pm0.2$ & $0.0\pm0.0$ & $77.8\pm0.4$ & $0.0\pm0.0$ & $62.9\pm0.4$ & $0.0\pm0.0$\tabularnewline
\cmidrule{2-8} \cmidrule{3-8} \cmidrule{4-8} \cmidrule{5-8} \cmidrule{6-8} \cmidrule{7-8} \cmidrule{8-8} 
\multirow{2}{*}{Robust CBC} & $83.7\pm0.1$ & $73.5\pm0.2$ & $60.0\pm0.2$ & $61.3\pm0.2$ & $\mathbf{35.0\pm0.2}$ & $45.0\pm0.1$ & $\mathbf{13.0\pm0.7}$\tabularnewline
 & $87.8\pm0.3$ & $77.4\pm0.3$ & $50.6\pm0.5$ & $62.8\pm0.3$ & $15.2\pm1.7$ & $45.3\pm0.4$ & $0.2\pm0.0$\tabularnewline
\cmidrule{2-8} \cmidrule{3-8} \cmidrule{4-8} \cmidrule{5-8} \cmidrule{6-8} \cmidrule{7-8} \cmidrule{8-8} 
\multirow{2}{*}{Robust RBF} & $83.6\pm0.2$ & $73.3\pm0.3$ & $59.7\pm0.2$ & $60.7\pm0.2$ & $34.9\pm0.3$ & $44.7\pm0.2$ & $12.7\pm0.4$\tabularnewline
 & $85.8\pm0.3$ & $74.0\pm0.5$ & $46.1\pm0.2$ & $58.3\pm0.6$ & $12.1\pm0.7$ & $41.2\pm0.5$ & $0.1\pm0.0$\tabularnewline
\cmidrule{2-8} \cmidrule{3-8} \cmidrule{4-8} \cmidrule{5-8} \cmidrule{6-8} \cmidrule{7-8} \cmidrule{8-8} 
\multirow{2}{*}{Robust CBC TD} & $89.8\pm0.1$ & $83.1\pm0.3$ & $\mathbf{61.0\pm0.1}$ & $73.7\pm0.2$ & $29.3\pm0.2$ & $59.3\pm0.1$ & $3.3\pm0.3$\tabularnewline
 & $91.9\pm0.3$ & $83.7\pm0.5$ & $40.7\pm0.7$ & $70.8\pm0.5$ & $1.6\pm0.2$ & $52.9\pm0.5$ & $0.0\pm0.0$\tabularnewline
\cmidrule{2-8} \cmidrule{3-8} \cmidrule{4-8} \cmidrule{5-8} \cmidrule{6-8} \cmidrule{7-8} \cmidrule{8-8} 
\multirow{2}{*}{Robust RBF TD} & $89.7\pm0.1$ & $82.8\pm0.2$ & $60.9\pm0.1$ & $73.4\pm0.1$ & $28.7\pm0.3$ & $59.2\pm0.1$ & $3.3\pm0.5$\tabularnewline
 & $91.5\pm0.3$ & $83.6\pm0.4$ & $38.4\pm0.4$ & $71.2\pm0.3$ & $1.0\pm0.1$ & $53.5\pm0.1$ & $0.0\pm0.0$\tabularnewline
\bottomrule
\end{tabular}
\par\end{centering}
\caption{Test, empirical robust, and certified robust accuracy of different
shallow prototype-based models with \emph{component-wise} temperatures
and robustness loss scaling of $\lambda=0.09$. The top shows prior
art, and the bottom shows our models. We put the best accuracy for
each category in bold. The top row always shows the results for the
non-squared distances, whereas the bottom row shows the results for
the squared distances.\label{tab:Shallow-results-long-individual-sigma}}
\end{table}
\begin{table}
\begin{centering}
\begin{tabular}{cccccccc}
\toprule 
 & \multirow{2}{*}{Accuracy} & \multicolumn{2}{c}{$\epsilon=0.5$} & \multicolumn{2}{c}{$\epsilon=1.0$} & \multicolumn{2}{c}{$\epsilon=1.58$}\tabularnewline
 &  & Emp.~Rob. & Cert.~Rob. & Emp.~Rob. & Cert.~Rob. & Emp.~Rob & Cert.~Rob.\tabularnewline
\midrule
\multirow{2}{*}{GLVQ} & $79.5\pm0.5$ & $70.2\pm0.5$ & $56.4\pm0.4$ & $58.8\pm0.3$ & $32.4\pm0.3$ & $44.1\pm0.3$ & $14.7\pm0.3$\tabularnewline
 & $80.5\pm0.6$ & $71.2\pm0.6$ & $57.0\pm0.4$ & $59.6\pm0.3$ & $32.3\pm0.3$ & $44.6\pm0.4$ & $14.0\pm0.4$\tabularnewline
\cmidrule{2-8} \cmidrule{3-8} \cmidrule{4-8} \cmidrule{5-8} \cmidrule{6-8} \cmidrule{7-8} \cmidrule{8-8} 
\multirow{2}{*}{RBF} & $89.0\pm0.2$ & $78.7\pm0.2$ & $-$ & $62.2\pm0.4$ & $-$ & $39.1\pm1.3$ & $-$\tabularnewline
 & $91.9\pm0.1$ & $81.6\pm0.1$ & $-$ & $61.2\pm1.2$ & $-$ & $28.7\pm1.2$ & $-$\tabularnewline
\cmidrule{2-8} \cmidrule{3-8} \cmidrule{4-8} \cmidrule{5-8} \cmidrule{6-8} \cmidrule{7-8} \cmidrule{8-8} 
\multirow{2}{*}{original CBC} & $71.9\pm5.9$ & $64.4\pm5.0$ & $-$ & $55.5\pm4.0$ & $-$ & $43.3\pm2.5$ & $-$\tabularnewline
 & $80.1\pm2.6$ & $71.7\pm1.8$ & $-$ & $61.3\pm0.8$ & $-$ & $46.8\pm0.4$ & $-$\tabularnewline
\cmidrule{2-8} \cmidrule{3-8} \cmidrule{4-8} \cmidrule{5-8} \cmidrule{6-8} \cmidrule{7-8} \cmidrule{8-8} 
\multirow{2}{*}{GTLVQ} & $92.9\pm0.2$ & $87.7\pm0.4$ & $81.7\pm0.6$ & $79.7\pm0.6$ & $63.4\pm0.8$ & $66.4\pm0.5$ & $\mathbf{37.6\pm0.5}$\tabularnewline
 & $94.0\pm0.2$ & $89.3\pm0.4$ & $\mathbf{83.1\pm0.7}$ & $81.5\pm0.6$ & $\mathbf{63.8\pm0.9}$ & $\mathbf{67.9\pm0.5}$ & $35.6\pm0.6$\tabularnewline
\cmidrule{2-8} \cmidrule{3-8} \cmidrule{4-8} \cmidrule{5-8} \cmidrule{6-8} \cmidrule{7-8} \cmidrule{8-8} 
\multirow{2}{*}{RBF TD} & $97.2\pm0.1$ & $92.9\pm0.1$ & $-$ & $\mathbf{83.1\pm0.2}$ & $-$ & $59.7\pm0.7$ & $-$\tabularnewline
 & \textbf{$\mathbf{97.9\pm0.1}$} & $\mathbf{93.0\pm0.2}$ & $-$ & $79.1\pm0.8$ & $-$ & $43.7\pm1.3$ & $-$\tabularnewline
\cmidrule{2-8} \cmidrule{3-8} \cmidrule{4-8} \cmidrule{5-8} \cmidrule{6-8} \cmidrule{7-8} \cmidrule{8-8} 
\multirow{2}{*}{original CBC TD} & $92.9\pm0.1$ & $87.5\pm0.1$ & $-$ & $79.3\pm0.2$ & $-$ & $65.8\pm0.2$ & $-$\tabularnewline
 & $95.0\pm0.3$ & $90.7\pm0.5$ & $-$ & $82.8\pm0.7$ & $-$ & $67.1\pm0.4$ & $-$\tabularnewline
\midrule
\multirow{2}{*}{CBC} & $82.6\pm0.6$ & $73.1\pm0.6$ & $59.8\pm0.8$ & $61.5\pm0.6$ & $\mathbf{36.0\pm0.8}$ & $46.9\pm0.5$ & $\mathbf{17.4\pm0.4}$\tabularnewline
 & $86.6\pm0.2$ & $78.0\pm0.3$ & $49.4\pm0.4$ & $66.6\pm0.2$ & $5.5\pm0.6$ & $50.9\pm0.2$ & $0.0\pm0.0$\tabularnewline
\cmidrule{2-8} \cmidrule{3-8} \cmidrule{4-8} \cmidrule{5-8} \cmidrule{6-8} \cmidrule{7-8} \cmidrule{8-8} 
\multirow{2}{*}{RBF-norm} & $72.0\pm0.3$ & $62.8\pm0.3$ & $48.5\pm0.3$ & $52.1\pm0.2$ & $27.1\pm0.3$ & $39.1\pm0.4$ & $12.9\pm0.2$\tabularnewline
 & $74.1\pm0.3$ & $64.8\pm0.3$ & $35.5\pm0.4$ & $53.9\pm0.2$ & $8.7\pm0.1$ & $40.4\pm0.2$ & $0.0\pm0.0$\tabularnewline
\cmidrule{2-8} \cmidrule{3-8} \cmidrule{4-8} \cmidrule{5-8} \cmidrule{6-8} \cmidrule{7-8} \cmidrule{8-8} 
\multirow{2}{*}{CBC TD} & $93.5\pm0.3$ & $88.4\pm0.5$ & $\mathbf{63.7\pm1.2}$ & $80.6\pm0.7$ & $16.0\pm0.7$ & $67.2\pm0.8$ & $0.0\pm0.0$\tabularnewline
 & $\mathbf{96.0\pm0.0}$ & $\mathbf{92.2\pm0.1}$ & $1.5\pm0.2$ & $\mathbf{84.9\pm0.2}$ & $0.0\pm0.0$ & $\mathbf{68.6\pm0.3}$ & $0.0\pm0.0$\tabularnewline
\cmidrule{2-8} \cmidrule{3-8} \cmidrule{4-8} \cmidrule{5-8} \cmidrule{6-8} \cmidrule{7-8} \cmidrule{8-8} 
\multirow{2}{*}{RBF-norm TD} & $90.0\pm0.2$ & $84.0\pm0.2$ & $51.1\pm0.2$ & $74.7\pm0.2$ & $15.4\pm0.4$ & $60.1\pm0.3$ & $0.0\pm0.0$\tabularnewline
 & $91.3\pm0.2$ & $85.6\pm0.4$ & $7.9\pm0.7$ & $76.6\pm0.4$ & $0.0\pm0.0$ & $61.7\pm0.7$ & $0.0\pm0.0$\tabularnewline
\cmidrule{2-8} \cmidrule{3-8} \cmidrule{4-8} \cmidrule{5-8} \cmidrule{6-8} \cmidrule{7-8} \cmidrule{8-8} 
\multirow{2}{*}{Robust CBC} & $83.5\pm0.5$ & $73.4\pm0.5$ & $59.9\pm0.3$ & $61.0\pm0.5$ & $35.1\pm0.6$ & $45.0\pm0.3$ & $13.0\pm0.5$\tabularnewline
 & $92.5\pm0.1$  & $77.9\pm0.9$  & $11.9\pm0.4$  & $54.2\pm0.7$  & $0.5\pm0.2$ & $31.4\pm0.4$ & $0.0\pm0.0$\tabularnewline
\cmidrule{2-8} \cmidrule{3-8} \cmidrule{4-8} \cmidrule{5-8} \cmidrule{6-8} \cmidrule{7-8} \cmidrule{8-8} 
\multirow{2}{*}{Robust RBF} & $83.6\pm0.3$ & $73.3\pm0.3$ & $59.7\pm0.1$ & $60.7\pm0.2$ & $35.0\pm0.3$ & $44.7\pm0.2$ & $12.8\pm0.4$\tabularnewline
 & $91.7\pm0.3$ & $78.5\pm1.1$ & $29.6\pm0.7$ & $58.5\pm1.5$ & $1.3\pm0.1$ & $35.3\pm0.9$ & $0.0\pm0.0$\tabularnewline
\cmidrule{2-8} \cmidrule{3-8} \cmidrule{4-8} \cmidrule{5-8} \cmidrule{6-8} \cmidrule{7-8} \cmidrule{8-8} 
\multirow{2}{*}{Robust CBC TD} & $89.6\pm0.2$ & $82.9\pm0.2$ & $61.0\pm0.2$ & $73.6\pm0.2$ & $29.3\pm0.3$ & $59.3\pm0.2$ & $3.1\pm0.3$\tabularnewline
 & $95.7\pm0.2$ & $89.3\pm0.4$ & $17.6\pm0.8$ & $76.6\pm0.5$ & $0.0\pm0.0$ & $55.3\pm0.7$ & $0.0\pm0.0$\tabularnewline
\cmidrule{2-8} \cmidrule{3-8} \cmidrule{4-8} \cmidrule{5-8} \cmidrule{6-8} \cmidrule{7-8} \cmidrule{8-8} 
\multirow{2}{*}{Robust RBF TD} & $89.8\pm0.2$ & $82.9\pm0.2$ & $61.1\pm0.2$ & $73.5\pm0.2$ & $29.1\pm0.4$ & $59.4\pm0.2$ & $3.4\pm0.4$\tabularnewline
 & $\mathbf{96.0\pm0.1}$ & $89.6\pm0.2$ & $13.4\pm1.1$ & $77.3\pm0.4$ & $0.0\pm0.0$ & $56.2\pm0.3$ & $0.0\pm0.0$\tabularnewline
\bottomrule
\end{tabular}
\par\end{centering}
\caption{Test, empirical robust, and certified robust accuracy of different
shallow prototype-based models with \emph{one} trainable temperature
shared between all components and robustness loss scaling of $\lambda=1$.
The top shows prior art, and the bottom shows our models. The top
row always shows the results for the non-squared distances, whereas
the bottom row shows the results for the squared distances. We put
the best accuracy for each category in bold.\label{tab:Shallow-results-long-one-sigma}}
\end{table}
If we compare RBF-norm with CBC for both the Euclidean and the tangent
distance in \tabref{Shallow-results-long-individual-sigma} and~\ref{tab:Shallow-results-long-one-sigma},
we can observe, in all cases, an accuracy increase. Therefore, we
conclude that negative reasoning fosters accuracy. The same trend
is almost always observed when we compare our CBC (or CBC TD) with
GLVQ (or GTLVQ). The only violation happens for non-squared distances
and component-wise temperatures (see \ref{tab:Shallow-results-long-individual-sigma}),
which can be explained by diverged components during training after
visual inspection. This underlines the importance of a suitable temperature
initialization and indicates that our proposed approach might be too
simplistic. 

In general, only RBF networks outperform our proposed approach. However,
it should be noted that, compared to GLVQ or our models, an RBF network
suffers from the mentioned interpretability shortcomings (see \secref{Review-of-Prototype-based}).
Moreover, squared distances achieve a higher accuracy.

\paragraph{Our CBC fixes the issue of original CBC.}

\begin{figure}
\begin{centering}
\includegraphics[width=0.9\textwidth]{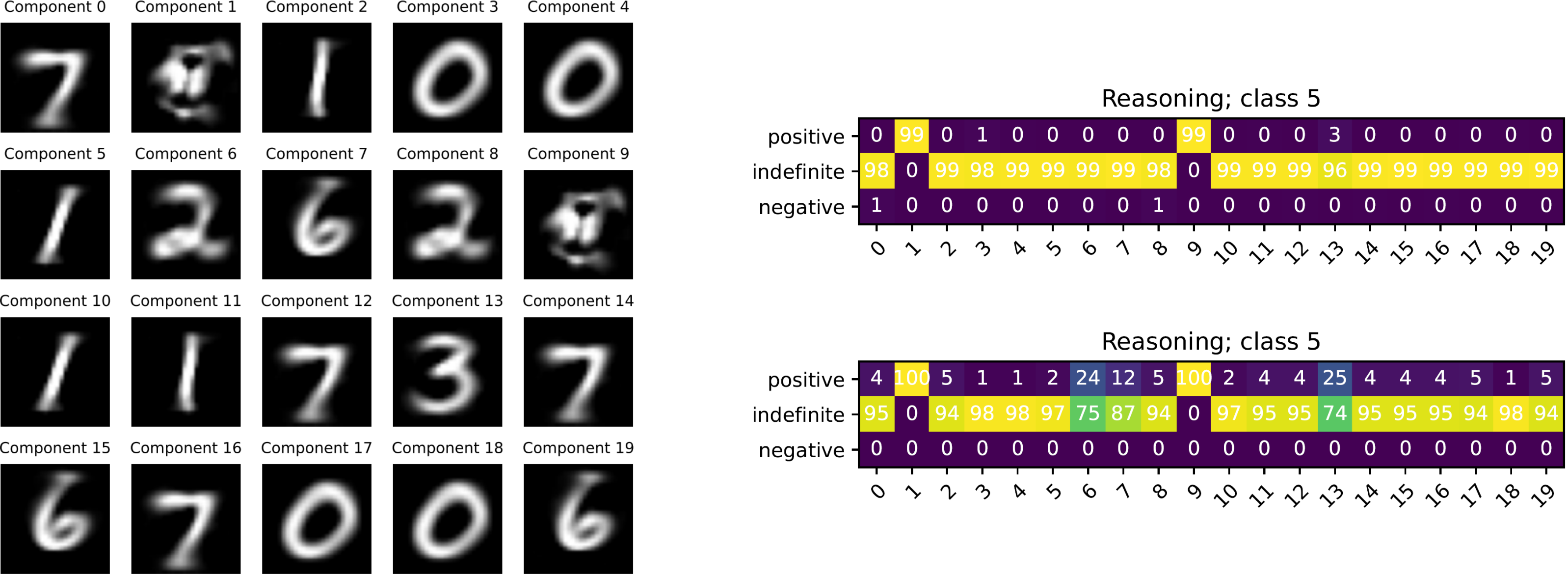}
\par\end{centering}
\caption{Learned reasoning of the original CBC. The reasoning matrix shows
for each component the learned probabilities. Please note that the
values displayed are rounded, which is why they may not add up to
exactly 100\,\%.\label{fig:Learned-reasoning-original-CBC}}
\end{figure}
\begin{figure}
\begin{centering}
\includegraphics[width=0.9\textwidth]{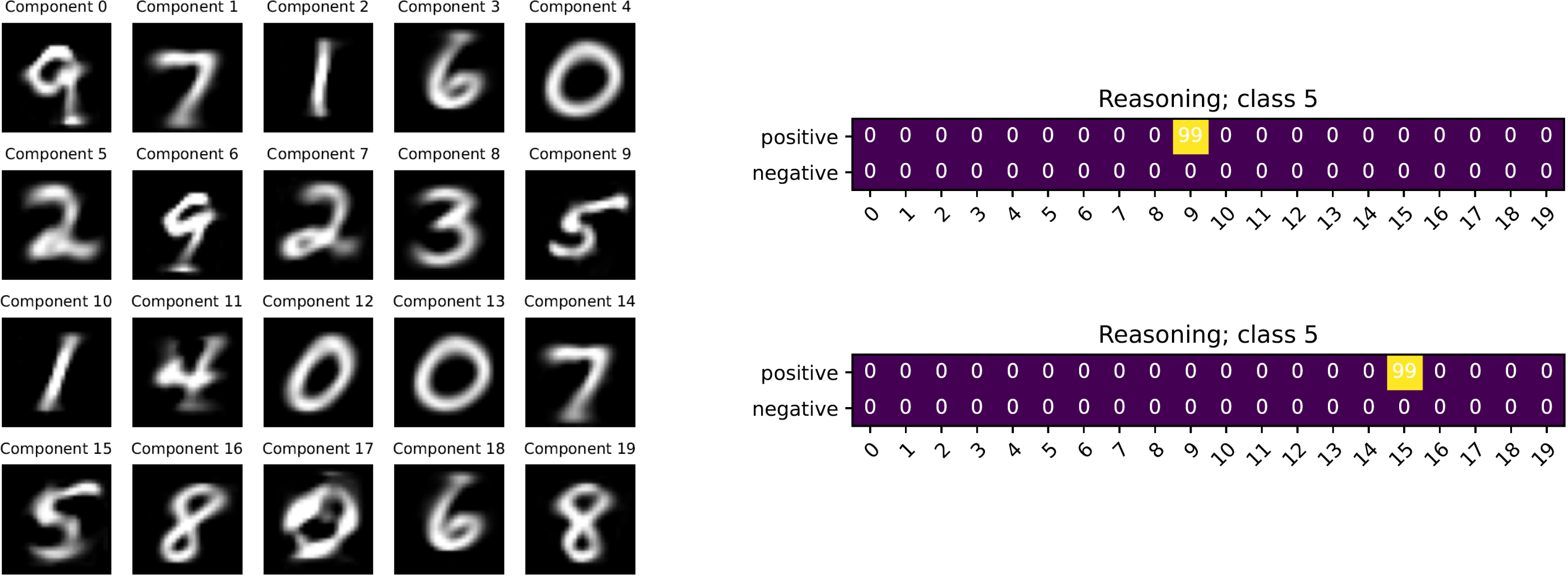}
\par\end{centering}
\caption{Learned reasoning of our CBC. The reasoning matrix shows for each
component the learned probabilities.\label{fig:Learned-reasoning-CBC}}
\end{figure}
 Both CBC variants are interpreted by analyzing the learned reasoning
probabilities and components. In the original CBC approach, the reasoning
consists of positive, negative, and indefinite. Moreover, the component
prior is set to be uniform over the number of components. Hence, following
\citet{Saralajew2019}, the reasoning is visualized in the form of
a matrix by showing the probabilities without the component prior.
Thus, the visualized probabilities in \figref{Learned-reasoning-original-CBC}
sum component-wise to 100\,\%. For our CBC, the reasoning probabilities
inherently depend on a usually non-uniform component prior. This results
in reasoning matrices where the sum of all values adds up to 100\,\%,
see \figref{Learned-reasoning-CBC}.

\figref{Learned-reasoning-original-CBC} and \ref{fig:Learned-reasoning-CBC}
show the reasoning concepts learned for the digit \emph{five}. For
the original CBC, we can see, by inspecting the components, that it
learned automatically class-specific components, but several components
are repetitions. For instance, all learned components that represent
a digit \emph{zero} are identical. Only for the digit \emph{one},
the components are different. Moreover, it should be noted that the
model has not learned class-specific components for the digits \emph{four},
\emph{five}, \emph{eight}, and \emph{nine}. So one question is how
does it differentiate between these classes. By analyzing the other
reasoning matrices for these digits, we can conclude that they are
almost identical. Therefore, over these four classes, the original
CBC does more or less random guessing. This also explains why the
accuracy is around 70\% (non-squared distance). Overall, this presented
example demonstrates the issue of the original CBC converging to bad
local minima.

On the other hand, our CBC learns different writing styles (concepts)
of a five and uses these concepts since there is one reasoning matrix
for each writing style of a five. However, even our model can learn
repetitions of components, as shown in \figref{Learned-reasoning-CBC}.
It should also be noted that our model learned a sparse representation
by only optimizing the margin loss \emph{without} any additional regularization
for sparsity. The reason why this happens can be explained by the
theoretical consideration about when the optimal output probabilities
are achieved, see \appendixref{Further-theoretical-results}, which
is exactly the case when the reasoning becomes crisp.

\begin{figure}
\begin{centering}
\includegraphics[width=0.35\textwidth]{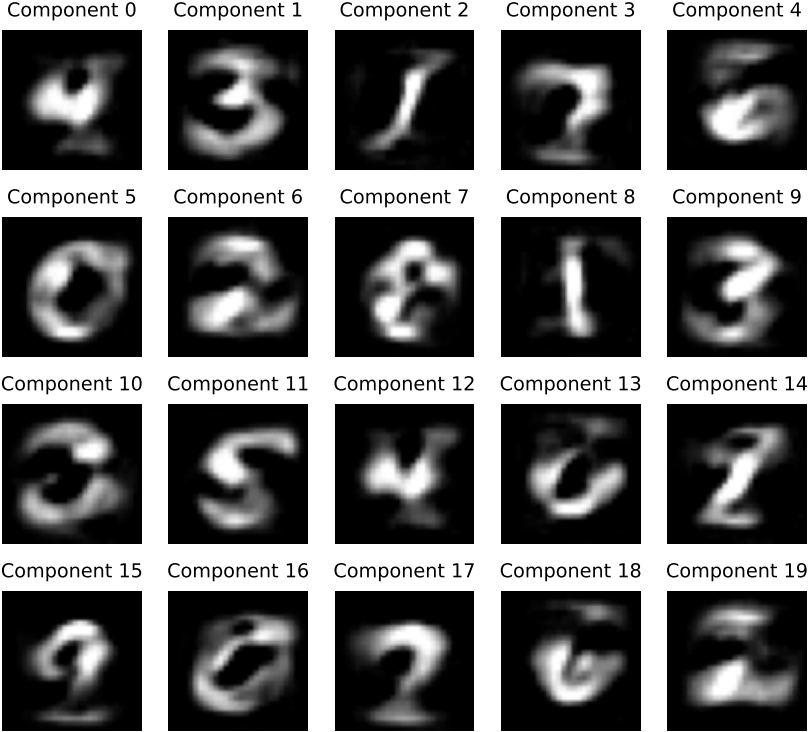}\hspace{4em}\includegraphics[width=0.35\textwidth]{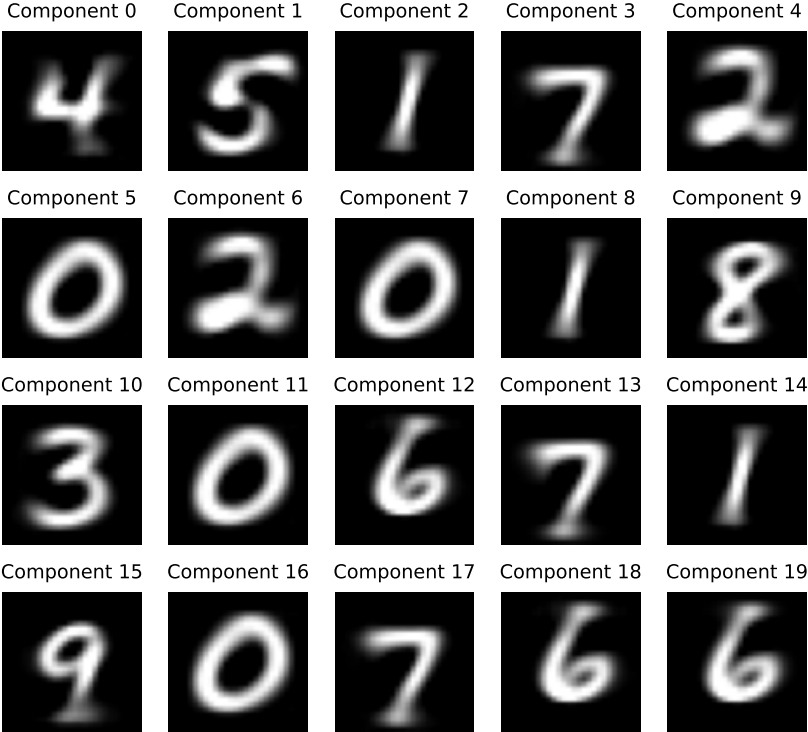}
\par\end{centering}
\caption{Components learned by an RBF (left) and by an RBF-norm (right).\label{fig:RBF_vs_RBF_norm}}
\end{figure}
If we analyze the effect of our interpretability constraint on RBF
networks, we see that the interpretability constraint promotes the
interpretability of the components, see \figref{RBF_vs_RBF_norm}.
However, we still encounter the issue of component repetitions even
if each class is represented. For squared distances, the results are
similar.

\paragraph{Advanced distance measures improve the accuracy.}

The results in \tabref{Shallow-results-long-individual-sigma} present
that an advanced dissimilarity function, such as the tangent distance,
constantly improves the performance of the classifiers. This underlines
that the selection of a suitable distance measure is of utmost importance
to build suitable classifiers with these shallow models.

\paragraph{Concept learning by shallow patch models.}

\begin{figure}
\begin{centering}
\includegraphics[width=1\textwidth]{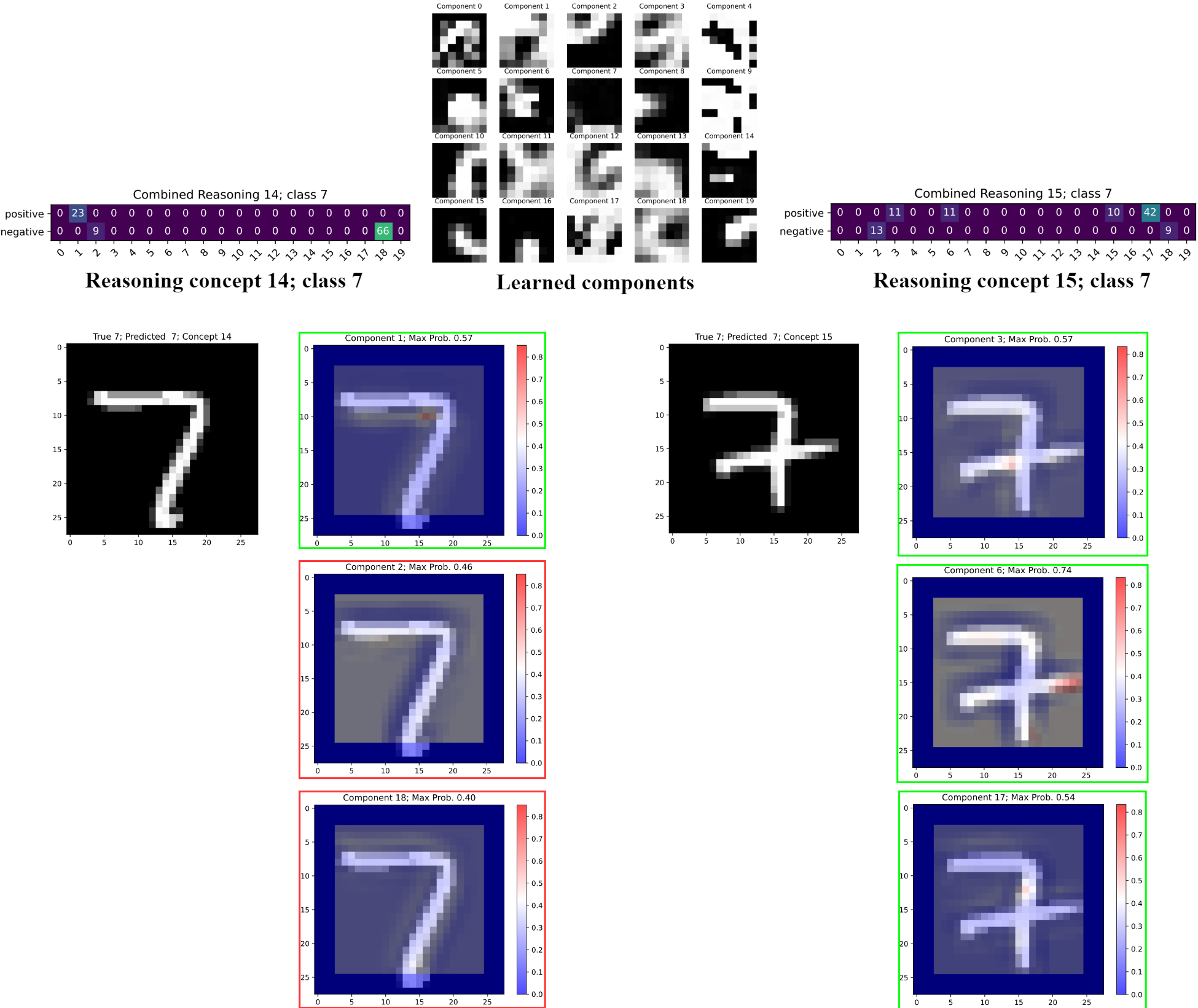}
\par\end{centering}
\caption{Visualizing the reasoning process of the two learned concepts for
the digit \emph{seven}. The method identified automatically to learn
one concept for the American \emph{seven} and one for the European
\emph{seven}. Components that are used for positive reasoning are
marked with green boxes and with red boxes otherwise.\label{fig:Patch_reasoning_shallow}}

\end{figure}
It is possible to build shallow patch prototype models. For this,
we train our CBC with a non-squared tangent distance of patch size
$7\times7$ and subspace dimension $4$. The tangent distance computation
is applied like a convolution operation so that we get the distance
responses at several pixel positions. After this, similar to deep
PBNs, we compute the pixel-wise similarity according to \eqref{rbf-kernel-definition}
and take the maximum over all pixel with respect to each component.
Then, we apply the reasoning probabilities. The entire network is
trained end-to-end and follows the training setting of the other shallow
models.

\figref{Patch_reasoning_shallow} shows the learned reasoning concepts.
In the middle, we see the learned translation vectors of the learned
affine subspaces. Because the components are affine subspaces, they
are, to some extent, transformation invariant so that small transformations
such as small rotations can be modeled. On the left and right, we
see the two learned reasoning concepts for the digit \emph{seven}.
Below each reasoning concept, we show a sample from the MNIST dataset
that is classified by this concept. Additionally, we show where the
components get activated in the input and highlight whether they are
used for positive reasoning. If we plot multiple correctly classified
samples for each concept, then the split between the American and
the European writing style of the digit seven becomes obvious.

With the components and the reasoning concepts, we can now interpret
the classification process: For the American seven, the CBC uses one
component for positive reasoning to detect the upper right corner
and two components that represent circles for the ``detection''
that no circles are in the input, for instance, to avoid confusions
with a nine. For the European seven, component 3 can detect the cross
in the middle of the seven. Additionally, component 6 detects whether
there is a left-sided line ending. Moreover, component 17 analyzes
if there is an upside-down ``T.'' Finally, the reasoning also checks
that there are no curved parts in the input.

By visualizing the components and reasoning probabilities in that
way, the method can be interpreted. Moreover, it can also be analyzed
why an input was incorrectly classified by visualizing the paths of
disagreement (see dashed paths in \figref{Probability-tree-diagram-CBC}),
which is related to visualizing the model confusion. This was already
used by \citet{Saralajew2019} to explain the success of an adversarial
attack.

\subsection{Robustness evaluation\label{appendix:Robustness-evaluation}}

\begin{table}
\begin{centering}
\begin{tabular}{cccccccc}
\toprule 
 & \multirow{2}{*}{Accuracy} & \multicolumn{2}{c}{$\epsilon=0.5$} & \multicolumn{2}{c}{$\epsilon=1.0$} & \multicolumn{2}{c}{$\epsilon=1.58$}\tabularnewline
 &  & Emp.~Rob. & Cert.~Rob. & Emp.~Rob. & Cert.~Rob. & Emp.~Rob & Cert.~Rob.\tabularnewline
\midrule
\multirow{2}{*}{Robust CBC} & $87.2\pm0.5$ & $75.8\pm1.0$ & $48.6\pm0.4$ & $60.6\pm1.1$ & $\mathbf{16.3\pm0.5}$ & $44.7\pm0.6$ & $\mathbf{0.2\pm0.1}$\tabularnewline
 & $92.5\pm0.1$ & $77.9\pm0.9$ & $11.9\pm0.4$ & $54.2\pm0.7$ & $0.5\pm0.2$ & $31.4\pm0.4$ & $0.0\pm0.0$\tabularnewline
\cmidrule{2-8} \cmidrule{3-8} \cmidrule{4-8} \cmidrule{5-8} \cmidrule{6-8} \cmidrule{7-8} \cmidrule{8-8} 
\multirow{2}{*}{Robust RBF} & $87.9\pm0.3$ & $77.9\pm0.5$ & $\mathbf{51.2\pm0.6}$ & $63.8\pm0.7$ & $15.5\pm1.0$ & $46.2\pm0.9$ & $\mathbf{0.2\pm0.0}$\tabularnewline
 & $91.7\pm0.3$ & $78.5\pm1.1$ & $29.6\pm0.7$ & $58.5\pm1.5$ & $1.3\pm0.1$ & $35.3\pm0.9$ & $0.0\pm0.0$\tabularnewline
\cmidrule{2-8} \cmidrule{3-8} \cmidrule{4-8} \cmidrule{5-8} \cmidrule{6-8} \cmidrule{7-8} \cmidrule{8-8} 
\multirow{2}{*}{Robust CBC TD} & $91.9\pm0.2$ & $83.7\pm0.5$ & $39.4\pm0.7$ & $71.3\pm0.4$ & $1.1\pm0.2$ & $53.5\pm0.4$ & $0.0\pm0.0$\tabularnewline
 & $95.7\pm0.2$ & $89.3\pm0.4$ & $17.6\pm0.8$ & $76.6\pm0.5$ & $0.0\pm0.0$ & $55.3\pm0.7$ & $0.0\pm0.0$\tabularnewline
\cmidrule{2-8} \cmidrule{3-8} \cmidrule{4-8} \cmidrule{5-8} \cmidrule{6-8} \cmidrule{7-8} \cmidrule{8-8} 
\multirow{2}{*}{Robust RBF TD} & $92.0\pm0.5$ & $84.5\pm0.8$ & $40.2\pm1.5$ & $72.4\pm1.0$ & $1.2\pm0.3$ & $54.6\pm1.0$ & $0.0\pm0.0$\tabularnewline
 & $\mathbf{96.0\pm0.1}$ & $\mathbf{89.6\pm0.2}$ & $13.4\pm1.1$ & $\mathbf{77.3\pm0.4}$ & $0.0\pm0.0$ & $\mathbf{56.2\pm0.3}$ & $0.0\pm0.0$\tabularnewline
\bottomrule
\end{tabular}
\par\end{centering}
\caption{Test, empirical robust, and certified robust accuracy of our different
robustified models with \emph{one} trainable temperature shared between
all components and squared distances. The top row always shows the
results for robustified models with $\lambda=0.09$, whereas the bottom
row shows the results $\lambda=1$. We put the best accuracy for each
category in bold.\label{tab:Shallow-results-long-one-sigma-unscaled-loss}}
\end{table}
In this experiment, we use the models from the previous section and
evaluate their robustness. Additionally, we compare the robustness
of our CBC with the robustified counterpart over a wide range of $\left\Vert \epsilon\right\Vert $
and margins and show how the robust loss training improves the robustness.

\paragraph{Robustness evaluation of all considered shallow models. }

\tabref{Shallow-results-long-individual-sigma} and~\ref{tab:Shallow-results-long-one-sigma}
present the full results of our robustness evaluation. First, it should
be noted that the trends observed in the main part of the paper are
also true for this larger set of validation models, which is that
the robustification leads to non-trivial certified robustness values.
Moreover, the robustified version frequently outperforms its non-robustified
counterparts with respect to empirical and certified robustness. For
instance, the robustified CBC outperforms the GLVQ model for non-squared
distances. This observation does not fully apply to the squared distances
since we see a drop in the certified and empirical robust accuracy.
This can be attributed to the additional lower bounding step, which
makes the derived loss (or equation for the certificate) less tight.
If we compare the robust CBCs with the robust RBF models, then we
see that the robust CBC scores are almost always slightly better than
the RBF. Similar to before, we attribute this observation to the effectiveness
of negative reasoning.

In \tabref{Shallow-results-long-one-sigma-unscaled-loss}, we present
the robustified accuracy for squared distance models trained with
$\lambda=1$ and $\lambda=0.09$. Since the two loss terms in the
robustified loss formulation are differently scaled, the loss terms
must be balanced by a regularization value. Usually, the loss term
$\delta$ for incorrect classification varies more than the loss term
for correct classification. Hence, promoting less incorrect classifications
if $\lambda=1$. The results in \tabref{Shallow-results-long-one-sigma-unscaled-loss}
present exactly this behavior. By removing the scaling, the model
achieves a higher accuracy but becomes less robust because more emphasis
is put on minimizing the number of incorrectly classified samples.

\paragraph{Robustness curves.}

\begin{figure}
\begin{centering}
\includegraphics[viewport=50bp 30bp 660bp 530bp,clip,width=0.45\textwidth]{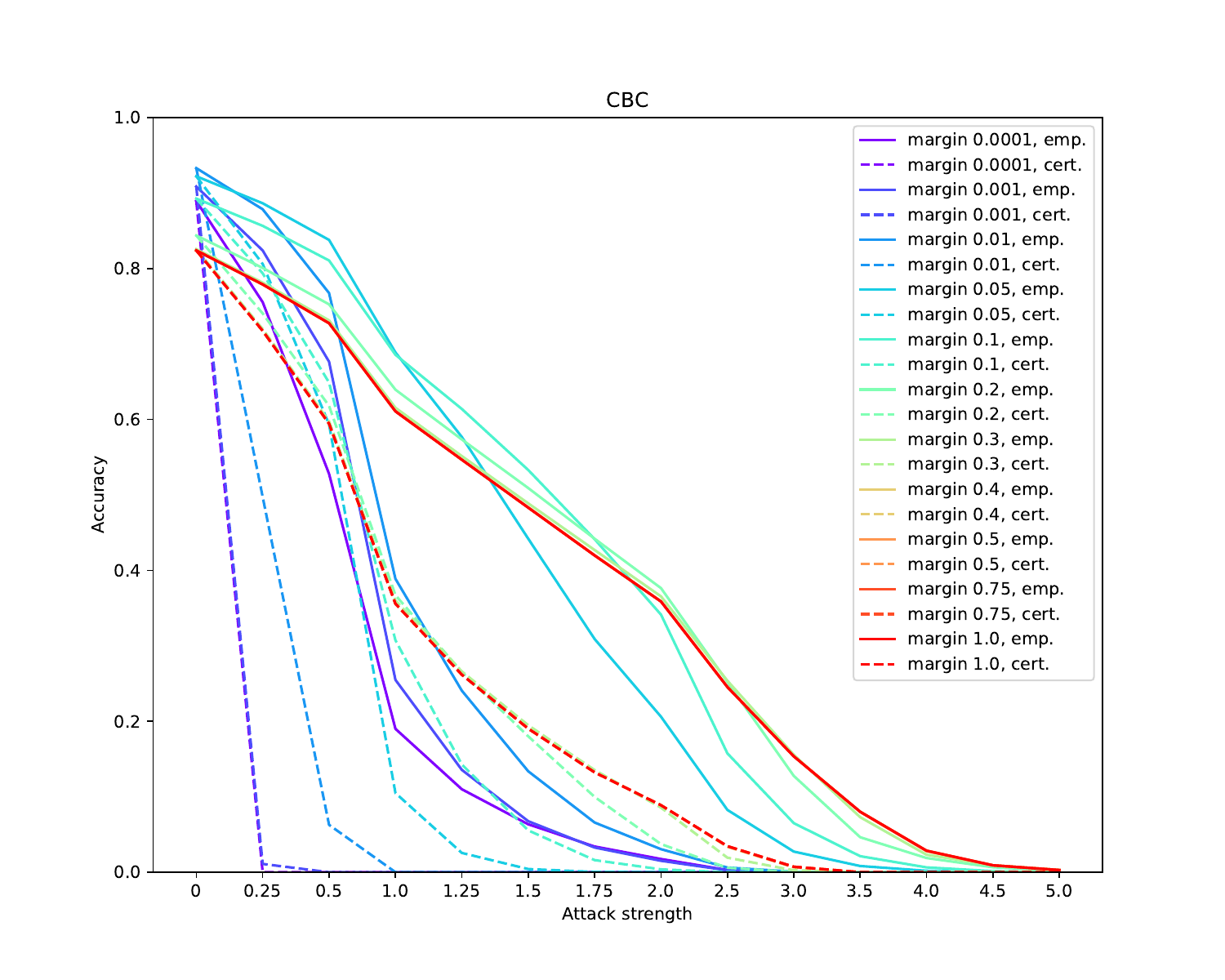}\includegraphics[viewport=50bp 30bp 660bp 530bp,clip,width=0.45\textwidth]{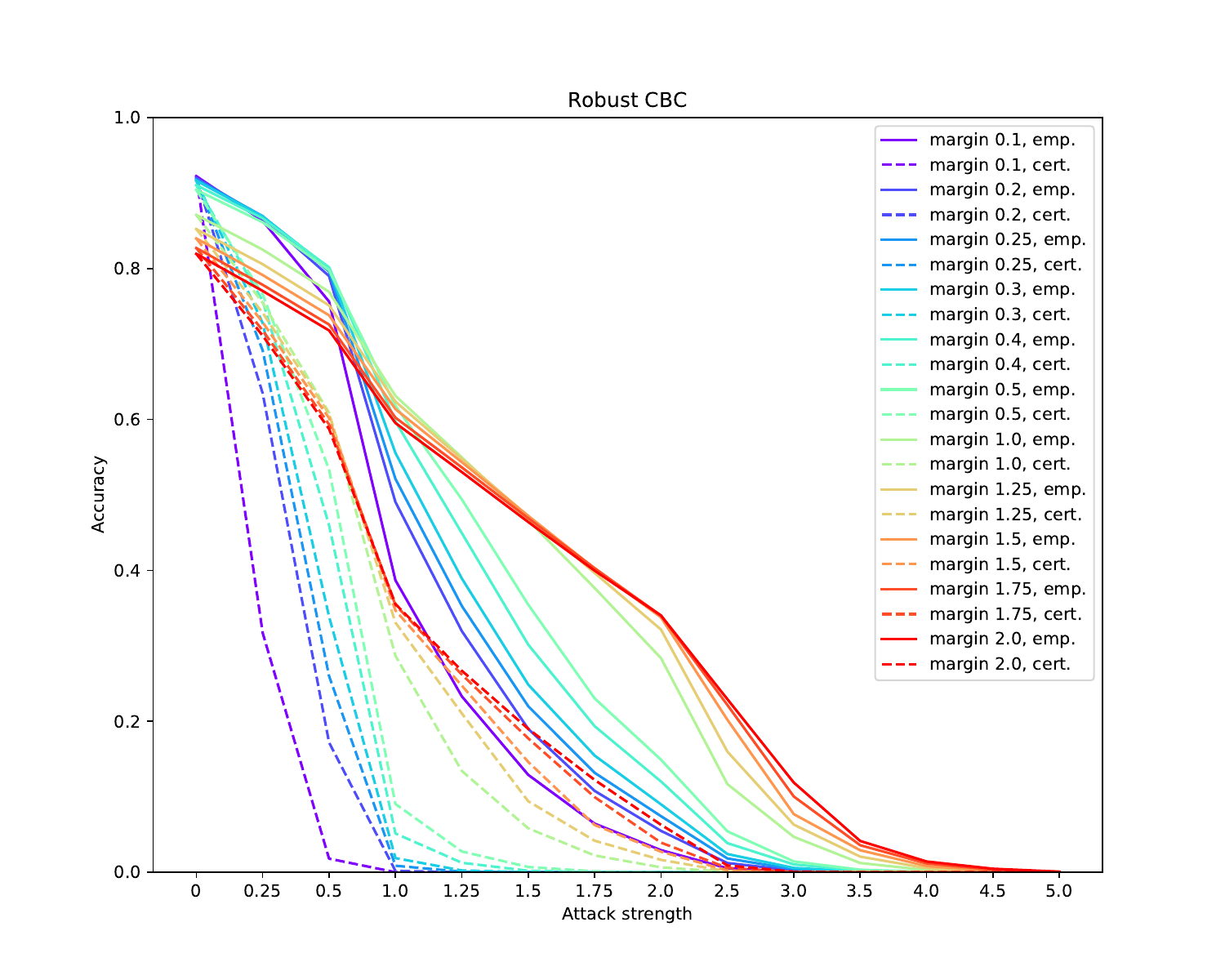}
\par\end{centering}
\caption{Robustness curves for non-squared CBC (left) and Robust CBC (right)
trained with different margins and evaluated for several $\left\Vert \boldsymbol{\epsilon}\right\Vert $.\label{fig:Robustness-curves}}
\end{figure}

\figref{Robustness-curves} presents the robustness curves of CBCs
and Robust CBCs with the non-squared Euclidean distance, one trainable
temperature, and $\lambda=1$. The curves for the Robust CBC show
how the optimization of the robust loss optimizes the certified robustness
accuracy. If the robustification margin is too small, then the model
is only provable robust for small attack strengths. With an increasing
robustness margin, the robust accuracy improves over the entire attack
strength range. However, this improved robustness lowers the clean
test accuracy. For CBC, we see that the model shows a similar empirical
robustness as Robust CBC for large margins. The maximum robustness
behavior is achieved for a margin of around $0.3$. After this value,
there is not much improvement in the empirical robustness. Even if
the network was not optimized for provable (certified) robustness,
the provable robustness is almost the same for large margins.

\subsubsection{Discussion}

Empirically, we observe that the maximization of the probability gap
also generates models with non-trivial certifiable robustness. Why
this happens has to be investigated. One possible explanation could
be that the output probability of a CBC model reaches its maximum
if the reasoning becomes crisp and, hence, becomes a GLVQ-like model.
At the same time, this implies that the loss reduces to the hypothesis
margin maximization (see \appendixref{Further-theoretical-results}).
To analyze this hypothesis, we determined whether the reasoning is
crisp when the model shows a non-trivial robustness. The collected
results showed that this hypothesis is not true since we found several
cases where the model was robust, but the reasoning was not crisp.
Another hypothesis we investigated is whether a larger probability
gap always increases the robustness. Again, this hypothesis must be
rejected for individual samples as it is easy to show that an \emph{individual}
sample can have a high margin but a small robustness. Additionally,
we analyzed whether this hypothesis holds over the entire dataset
on average. For this, we created a linear separable dataset and trained
1000 models that solved this dataset perfectly. For each model, we
computed the average probability gap and robustness loss and checked
whether they were ranked similarly. Again, we have to reject this
hypothesis (analyzed with Kendall tau rank loss)---we also found
several situations on MNIST where this hypothesis is violated. Consequently,
right now, it is an open problem why the probability gap maximization
encourages robustness. 

Also note that the true robustness of our created models might be
significantly higher because the AutoAttack framework consists of
strong attacks that approximate the true robustness well for shallow
prototype-based models \citep{Voracek2022}. Assuming this is true
raises the question of why “non-robustified” models such as RBF and
the original CBC achieve good robustness scores.

\end{document}